%% file: AAAMAIN.tex
\documentclass{article}

\usepackage{microtype}
\usepackage{graphicx}
\usepackage{subcaption}
\usepackage{booktabs} 

\usepackage{hyperref}

\PassOptionsToPackage{table}{xcolor}

\usepackage[preprint]{icml2026}
\usepackage{makecell}

\usepackage{tikz}
\usetikzlibrary{calc, shadows, arrows.meta, shapes}

\usepackage{amsmath}
\usepackage{amssymb}
\usepackage{mathtools}
\usepackage{amsthm}
\usepackage{colortbl}
\definecolor{sota-blue}{RGB}{235, 245, 255} 

\usepackage[capitalize,noabbrev]{cleveref}

\theoremstyle{plain}
\newtheorem{theorem}{Theorem}[section]
\newtheorem{proposition}[theorem]{Proposition}

\theoremstyle{definition}

\theoremstyle{remark}
\newtheorem{remark}[theorem]{Remark}
\usepackage{multirow}
\usepackage{enumitem}
\usepackage{booktabs}  
\usepackage{amssymb}   
\usepackage[most]{tcolorbox} 

\newtcolorbox[auto counter, number within=section]{thmbox}[2][]{%
  enhanced,
  colback=gray!3!white,
  colframe=black,
  sharp corners,
  boxrule=0.8pt,
  left=8pt, right=8pt, top=8pt, bottom=8pt,
  title={Theorem~\thetcbcounter~(#2).},
  attach title to upper={\quad},
  coltitle=black,
  fonttitle=\bfseries
  ,#1
}

\usepackage[textsize=tiny]{todonotes}

\icmltitlerunning{Spectral Gating Networks}

\begin{document}

\twocolumn[
  \icmltitle{Spectral Gating Networks}
  \icmlsetsymbol{equal}{*}

  \begin{icmlauthorlist}
    \icmlauthor{Jusheng Zhang}{sysu,ntu}
    \icmlauthor{Yijia Fan}{sysu}
    \icmlauthor{Kaitong Cai}{sysu}
    \icmlauthor{Jing Yang}{sysu} \\
    \icmlauthor{Yongsen Zheng}{ntu}
    \icmlauthor{Kwok-Yan Lam}{ntu}
    \icmlauthor{Liang Lin}{sysu}
    \icmlauthor{Keze Wang}{sysu}
  \end{icmlauthorlist}
  \icmlaffiliation{sysu}{Sun Yat-sen University, China}
  \icmlaffiliation{ntu}{Nanyang Technological University, Singapore}

  \icmlcorrespondingauthor{Keze Wang}{kezewang@gmail.com}

  \icmlkeywords{Machine Learning, ICML, Rational ANOVA Networks}

  \vskip 0.3in
]
\printAffiliationsAndNotice{}  
\begin{abstract}
Gating mechanisms are ubiquitous, yet a complementary question in feed-forward networks remains under-explored, i.e., how to introduce frequency-rich expressivity without sacrificing stability and scalability? This tension is exposed by spline-based Kolmogorov--Arnold Network (KAN) parameterizations, where grid refinement can induce parameter growth and brittle optimization in high dimensions.
To propose a stability-preserving way to inject spectral capacity into existing MLP/FFN layers under fixed parameter and training budgets, we introduce \emph{Spectral Gating Networks (SGN)}, a drop-in spectral reparameterization. Our SGN augments a standard activation pathway with a compact spectral pathway and learnable gates that allow the model to start from a stable base behavior and progressively allocate capacity to spectral features during training. The spectral pathway is instantiated with trainable Random Fourier Features (learned frequencies and phases), replacing grid-based splines and removing resolution dependence. A hybrid GELU-Fourier formulation further improves optimization robustness while enhancing high-frequency fidelity.
Across vision, NLP, audio, and PDE benchmarks, SGN consistently improves accuracy-efficiency trade-offs under comparable computational budgets, achieving \textbf{93.15\%} accuracy on CIFAR-10 and up to \textbf{11.7$\times$} faster inference than spline-based KAN variants. Code and trained models will be released.
\end{abstract}

\input{sec/1_intro}
\input{sec/2_Related_Work}

\input{sec/3_Method}
\input{sec/4_Experiment}
\input{sec/5_Analysis}
\input{sec/6_Conclusion}

\bibliography{example_paper}
\bibliographystyle{icml2026}
\newpage
\appendix
\onecolumn
\input{sec/X_suppl}

\end{document}

%% file: sec/1_intro.tex
\section{Introduction}
\label{sec:intro}

Spectral gating is well-established in neural networks, but its role inside feed-forward networks (MLP) ~\cite{MLP,Attention,qiankui,qiankui2} remains surprisingly under-examined. Early architectures introduced learnable nonlinearity shaping through fixed activations and gating, and modern models continue to rely on FFN blocks as the primary source of feature transformation and capacity ~\cite{srivastava2015highwaynetworks,dauphin2017languagemodelinggatedconvolutional,mk,Attention,zsfly}. In parallel, learnable basis parameterizations—ranging from spline families to Fourier-feature representations—have been repeatedly proposed to enhance expressiveness, especially for capturing high-frequency components~\cite{799930,tancik2020fourierfeaturesletnetworks}. Despite widespread adoption and empirical success, the function and impact of ``spectral capacity injection'' in FFNs are still insufficiently understood beyond intuition, making it hard to assess what truly drives improvements once it is entangled with other architectural changes.

Insufficient understanding is especially problematic when basis changes are confounded with other factors~\cite{zsbiaozhen,zsfly}. For instance, spline-parameterized designs inspired by Kolmogorov--Arnold Networks~\cite{kan,kan2} may report stronger approximation, yet their performance and efficiency often co-vary with grid resolution, parameter growth, normalization choices, and optimization recipes, i.e., making it unclear whether gains come from the spline basis itself or from accompanying design or engineering decisions. Similarly, Fourier-feature augmentations can improve spectral expressiveness, but are frequently introduced together with width or depth changes, normalization, or task-specific tuning, which again obscures the contribution of the spectral representation. These considerations underscore the need to rigorously disentangle the effects of spectral expressivity from stability and scaling factors in FFNs, and to identify the minimal modification that yields consistent benefits across settings.
\begin{figure*}[t]
  \centering
  \includegraphics[width=\textwidth]{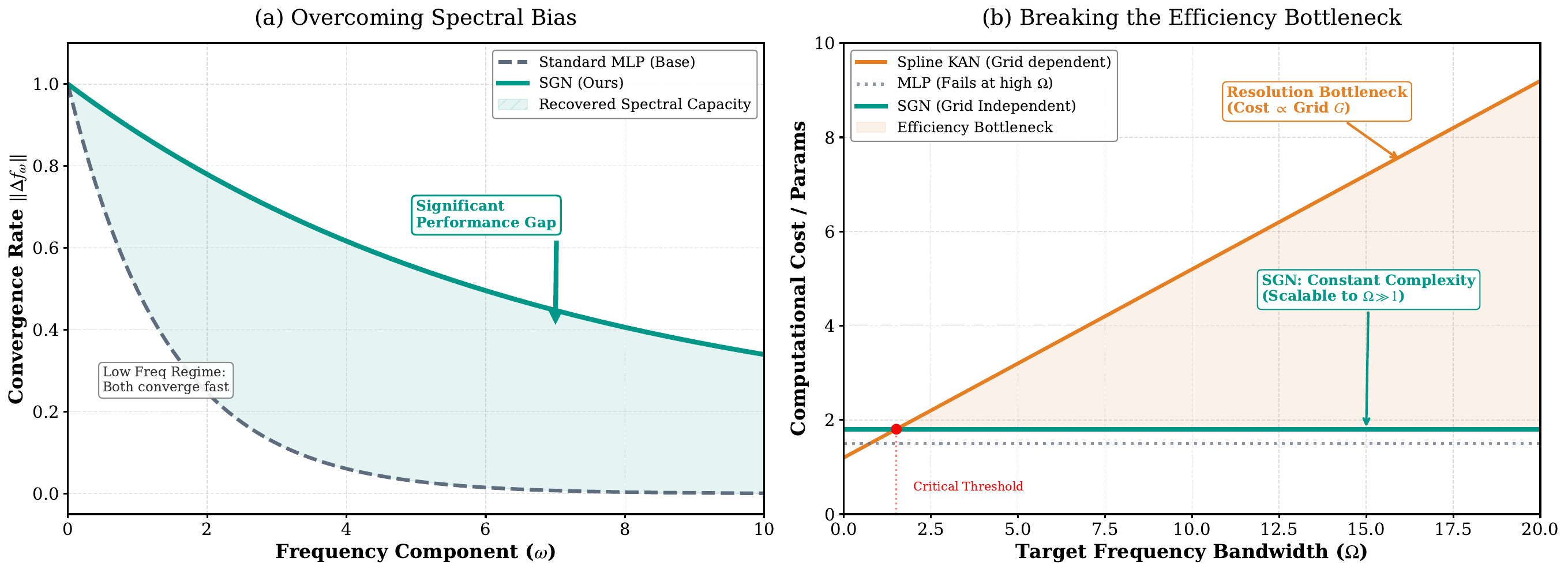}
  \vspace{-15pt}
  \caption{\textbf{Conceptual illustration and theoretical analysis of the Spectral-Efficiency gap.} 
  \textbf{(a) Overcoming Spectral Bias:} Illustrating \textit{Problem 1}. Standard MLPs (grey dashed) suffer from exponential convergence decay as frequency $\omega$ increases, creating a "spectral gap" (hatched area). SGN (green solid) maintains robust learnability across the spectrum.
  \textbf{(b) Breaking the Efficiency Bottleneck:} Illustrating \textit{Problem 2}. While spline-based KANs (orange) incur linear cost scaling with grid resolution $G$ (creating an efficiency bottleneck at high $\Omega$), SGN maintains constant complexity, effectively enabling scalable high-frequency modeling.}
  \label{fig:spectral_efficiency_gap}
  \vspace{-10pt}
\end{figure*}
In this work, we investigate spectral gating mechanisms in standard FFNs. Specifically, we introduce a \textbf{Spectral Gating Network (SGN)} module~\cite{MLP,ZSFFN} (formally defined in \textbf{Eq.~\ref{eq:sgn-core}}, Sec.~\ref{sec:method_sgn}) that augments a conventional FFN with an additional compact spectral pathway and a learnable gate that controls information flow between pathways. Our exploration covers SGN variants along several key aspects: (1) spectral bases (fixed vs.\ trainable Random Fourier Features; learned frequencies and phases), (2) gating forms (additive vs.\ multiplicative; scalar vs.\ channel-wise), and (3) stabilization strategies (where and how to normalize the spectral branch; how to initialize gates to preserve a stable ``base FFN'' regime). The exploration results demonstrate that: (i) a gated spectral augmentation is consistently beneficial, and the most reliable configuration is one that keeps the original activation pathway intact while letting the spectral branch be introduced progressively via a learnable gate; (ii) SGN substantially improves the accuracy--efficiency trade-off across modalities, achieving strong performance under comparable budgets while avoiding the grid-resolution--driven parameter and latency overhead that often arises in spline-table parameterizations (see theoretical complexity in \textbf{Theorem~\ref{thm:complexity}}).

We identify two primary factors contributing to SGN's effectiveness. (i) \textbf{Stability-preserving decoupling.} The gate decouples ``optimization-friendly behavior'' from ``spectral capacity.'' By implementing a \textbf{continuation principle} (analyzed in \textbf{Theorem~\ref{thm:homotopy}}), the FFN starts from a stable base behavior and only allocates additional spectral expressivity when optimization supports it, preventing early-training brittleness that commonly accompanies aggressive basis replacement. (ii) \textbf{Adaptive spectral expressiveness.} Allowing the spectral branch to learn task-dependent frequencies and phases concentrates capacity on relevant bands, improving high-frequency fidelity without requiring grid refinement or large parameter tables. Together, these factors explain why SGN acts as a minimal, architecture-compatible modification that reliably improves both performance and scalability.

In summary, our study highlights the intrinsic value of spectral gating inside FFNs. By systematically evaluating SGN variants and isolating confounding factors, we show that a simple gated spectral augmentation can deliver consistent gains across vision, language, audio, and partial differential equation (PDE) tasks (comprehensive results in \textbf{Sec.~\ref{sec:experiments}}), while remaining close to MLP-level runtime.

%% file: sec/2_Related_Work.tex
\begin{figure*}[t]
  \centering
  \includegraphics[width=0.9\textwidth]{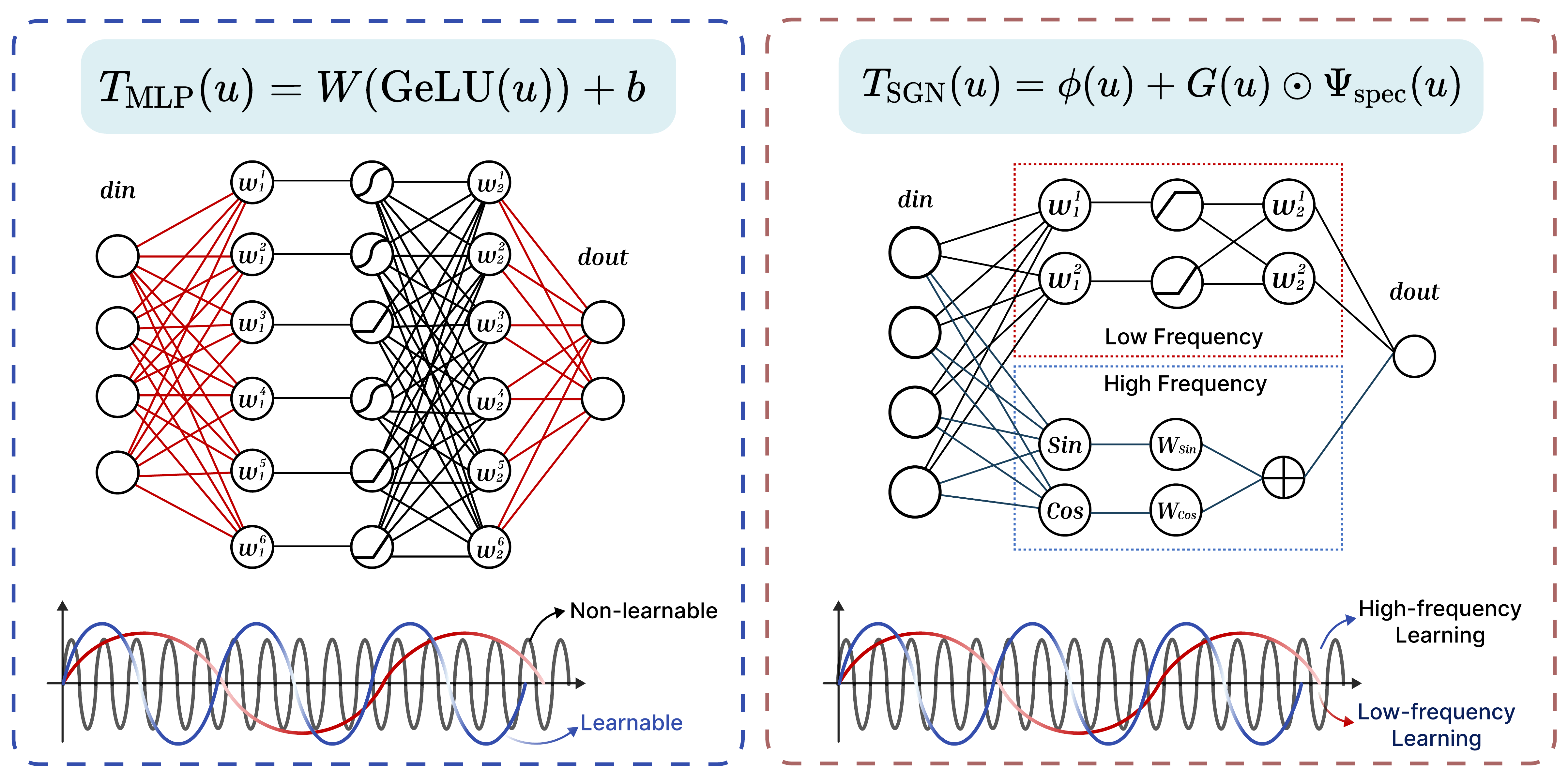} 
  \vspace{-5pt}
  \caption{\textbf{Architecture Comparison: MLP vs. SGN.} 
  (Left) A standard MLP layer relies on global weights and fixed activations, primarily capturing low-frequency components. 
  (Right) The proposed SGN architecture introduces a parallel \textbf{Spectral Branch}. 
  Features are split into a stable low-frequency path (Base Activation) and a high-frequency path (Sine/Cosine modulations via RFF). 
A learnable gate adaptively combines the two branches, achieving expressive detail modeling without sacrificing MLP stability.}
  \label{fig:sgn_architecture}
  \vspace{-5pt}
\end{figure*}
\section{Problem Formulation: The Spectral-Efficiency Gap}
\label{sec:problem_formulation}

In this section, we formalize the fundamental limitations of standard Feed-Forward Networks (FFNs ~\cite{MLP,qiankui}) in capturing high-frequency features, and analyze why recent spline-based alternatives (e.g., KAN-style parameterizations) do not yet provide a scalable drop-in solution for large backbones.

\textbf{Setup.}
Consider a standard FFN layer, the backbone of modern Transformers. Given an input token $x \in \mathbb{R}^{d_{\mathrm{model}}}$, the FFN computes:
\begin{equation}
\label{eq:standard_ffn}
    u = W_1 x + b_1, \quad y = W_2 \,\phi(u) + b_2,
\end{equation}
where $W_1 \in \mathbb{R}^{d_{\mathrm{ff}} \times d_{\mathrm{model}}}$, $W_2 \in \mathbb{R}^{d_{\mathrm{model}} \times d_{\mathrm{ff}}}$, and $\phi(\cdot)$ is a fixed element-wise nonlinearity (e.g., GELU). While efficient and stable, this parameterization faces a critical dilemma: improving high-frequency expressivity typically comes at the cost of either (i) scaling depth/width aggressively, or (ii) adopting alternative parameterizations that break hardware efficiency and warm-start compatibility.

\paragraph{Problem 1: The Intrinsic Spectral Bias of MLPs.}
The primary bottleneck of standard FFNs is their difficulty in efficiently learning high-frequency components. As characterized by the \emph{Frequency Principle}, gradient-based training of standard MLPs tends to fit low-frequency components earlier and more reliably than high-frequency ones. Concretely, consider a target function $f$ with Fourier spectrum $\hat{f}(\omega)$. For many common activations, the optimization dynamics exhibit a convergence behavior whose effective speed degrades as $\|\omega\|$ grows \textbf{(visualized in Fig.~\ref{fig:spectral_efficiency_gap}a)}. This implies that components at large $\|\omega\|$ are learned late or remain underfit within a practical compute budget~\cite{qiankui,qiankui2,ZSFFN2}.

This spectral bias becomes a practical issue in tasks where fine-grained details are essential (e.g., textures in vision, or high-frequency oscillations in PDE solving)~\cite{geirhos2022imagenettrainedcnnsbiasedtexture}. In such cases, a standard FFN often compensates by excessively scaling depth or width, resulting in over-parameterized models that are still prone to underfitting high-frequency content.

\paragraph{Problem 2: The Resolution-Efficiency Bottleneck of Spline-based Solutions.}
A natural response to spectral bias is to introduce learnable function bases with higher local flexibility, such as KANs~\citep{kan}, which replace fixed activations with learnable splines defined over a grid of resolution $G$. While expressive, they introduce a severe \textbf{Resolution-Efficiency Trade-off} \textbf{(analyzed in Fig.~\ref{fig:spectral_efficiency_gap}b)}:
\begin{itemize}[leftmargin=1.2em, nosep]
    \item \textbf{Channel-wise Scaling:} Even for simple 1D splines, the cost scales as $O(d_{\mathrm{ff}} \cdot G)$. This couple's high-frequency capacity directly to grid resolution.
    \item \textbf{Hardware Inefficiency:} Spline evaluation often incurs irregular memory access (lookup/interpolation), significantly reducing throughput on modern accelerators compared to dense GEMM operations.
\end{itemize}
Consequently, spline-based alternatives are challenging to deploy as universally scalable drop-in replacements for FFNs in large-scale backbones.

\textbf{Our Goal.}
We aim to design a \textbf{universal FFN reparameterization} that resolves spectral bias \emph{without} incurring the efficiency penalties of spline grids. Specifically, we seek an operator $\mathcal{T}(u)$ such that:
\textbf{High-Frequency Capability:} Overcome MLP spectral bias and efficiently model bandwidth $\Omega \gg 1$.
\textbf{MLP-Level Efficiency:} Maintain hardware-friendly dense linear algebra with complexity linear in $d_{\mathrm{ff}}$, strictly independent of grid resolution.
\textbf{Warm-Start Compatibility:} Admit a stable initialization where $\mathcal{T}(u)\approx \phi(u)$, inheriting the optimization landscape of pretrained FFNs to avoid cold-start instability.

\section{Spectral Gating Networks (SGN)}
\label{sec:method_sgn}

To achieve the goals in Sec.~\ref{sec:problem_formulation}, we propose \textbf{Spectral Gating Networks (SGN)}.
Unlike spline-based replacements that \emph{discard} the standard activation (sacrificing stability)~\cite{yangtiao2,yangtiao}, SGN \textbf{augments} the standard FFN. We treat the base MLP as a stable low-frequency learner and inject a lightweight, gated spectral branch to learn high-frequency residuals~\cite{kaiming}.

\subsection{Formulation}
\label{sec:formulation}

\paragraph{Core Formulation.}
We redefine the FFN activation path in Eq.~\eqref{eq:standard_ffn} as a hybrid composition. For the latent feature $u \in \mathbb{R}^{d_{\mathrm{ff}}}$, SGN computes:
\begin{equation}
\label{eq:sgn-core}
    \mathcal{T}_{\text{SGN}}(u)
    = \underbrace{\phi(u)}_{\text{Base MLP Branch}}
    \;+\;
    \underbrace{\mathcal{G}(u) \odot \Psi_{\mathrm{spec}}(u)}_{\text{Spectral Branch}},
\end{equation}
where $\phi(u)$ is the standard activation (preserving pretrained behavior), $\Psi_{\mathrm{spec}}(u)$ is an adaptive spectral mapping, and $\mathcal{G}(u)$ is a learnable gate. Note that, unlike prior spectral or spline-based designs that replace the base activation, our SGN preserves the original FFN pathway and restricts the spectral component to a gated residual, which is critical for stability in pretrained or deep architectures.

This design directly addresses the problems in Sec.~\ref{sec:problem_formulation}:
\textbf{Solving Problem 1 (Spectral Bias):} The mapping $\Psi_{\mathrm{spec}}$ provides an explicit pathway for high-frequency gradients that fixed $\phi(\cdot)$ tends to suppress~\cite{tancik2020fourierfeaturesletnetworks}.
\textbf{Solving Problem 2 (Efficiency):} We parameterize $\Psi_{\mathrm{spec}}$ via dense Fourier features rather than grids, decoupling expressivity from resolution $G$ and preserving GEMM-friendly computation.

\begin{thmbox}{Homotopy Consistency \& Cold-Start Stability}
\label{thm:homotopy}
Let the FFN output be $y = W_2\,\mathcal{T}(u)+b_2$. Consider SGN with $\mathcal{T}_{\text{SGN}}$ in Eq.~\eqref{eq:sgn-core}. If the spectral branch is initialized with negligible magnitude (e.g., $A_r \sim \mathcal{N}(0,\epsilon^2)$), then at initialization:
\begin{equation}
\label{eq:homotopy_value}
\begin{aligned}
\mathcal{T}_{\text{SGN}}^{(0)}(u) &= \phi(u) + O(\epsilon), \\
\nabla_{\theta_{\text{base}}}\mathcal{L}_{\text{SGN}}^{(0)} &= \nabla_{\theta_{\text{base}}}\mathcal{L}_{\text{MLP}}^{(0)} + O(\epsilon), \\
\big\|\nabla_{\theta_{\text{spec}}}\mathcal{L}_{\text{SGN}}^{(0)}\big\| &= O(\epsilon),
\end{aligned}
\end{equation}
where $\theta_{\text{spec}}=\{W_r,b_r,A_r,w_g,b_g\}$.
This homotopy-style initialization guarantees that SGN inherits the robust optimization behavior of the pretrained FFN, avoiding the cold-start instability of pure spectral replacements.
\textbf{See Appendix~\ref{sec:proof_thm_homotopy} for the complete proof.}
\end{thmbox}

\paragraph{Lightweight Channel-wise Gating.}
To ensure SGN remains efficient, we employ \textbf{channel-wise gating}:
\begin{equation}
\label{eq:gate}
    \mathcal{G}(u) = \sigma\left( w_g \odot \mathrm{LN}(u) + b_g \right),
\end{equation}
where $w_g, b_g \in \mathbb{R}^{d_{\mathrm{ff}}}$ and $\mathrm{LN}$ denotes LayerNorm. 
This adds negligible parameter overhead: $2d_{\mathrm{ff}}$ parameters for $(w_g,b_g)$, plus an optional $2d_{\mathrm{ff}}$ if LayerNorm uses affine parameters $(\gamma,\beta)$, which still remains $O(d_{\mathrm{ff}})$. 
In contrast, full-matrix gating would incur $O(d_{\mathrm{ff}}^2)$ parameters, while our design preserves the per-channel adaptivity crucial for selective spectral injection.

\paragraph{Spectral Branch via Trainable RFF.}
We instantiate $\Psi_{\mathrm{spec}}$ using \textbf{Trainable Random Fourier Features (RFF)}. Let $m$ be a fixed \emph{spectral budget}. We define the feature map:
\begin{equation}
\label{eq:rff_map}
    \gamma(u) = \sqrt{\frac{2}{m}}
    \Big[
        \cos(W_r^\top u + b_r) \oplus \sin(W_r^\top u + b_r)
    \Big],
\end{equation}
where $W_r \in \mathbb{R}^{d_{\mathrm{ff}} \times m}$ and $b_r \in \mathbb{R}^{m}$. The spectral output is projected back to $\mathbb{R}^{d_{\mathrm{ff}}}$:
\begin{equation}
\label{eq:rff_proj}
    \Psi_{\mathrm{spec}}(u) = \gamma(u)\,A_r, \quad A_r \in \mathbb{R}^{2m \times d_{\mathrm{ff}}}.
\end{equation}
We initialize $W_r \sim \mathcal{N}(0,\sigma^2)$ to match a target bandwidth. Crucially, by \emph{learning} $W_r$, SGN adapts its spectral receptive field to the data distribution, achieving high-frequency fidelity without parameter explosion.

\subsection{Theoretical Analysis}
\label{sec:analysis}
We analyze how SGN preserves stability while expanding spectral capacity.
\textbf{Analysis: Gradient Modulation \& Spectral Expansion.}
To explain the mechanism, we decompose the Jacobian $J_{\mathcal{T}}(u)$ using the product rule:
\begin{equation}
\label{eq:jacobian_decomp}
\resizebox{0.9\hsize}{!}{$
J_{\mathcal{T}}(u)
=
\underbrace{\mathrm{diag}\big(\phi'(u)\big)}_{\text{Base Term}}
+
\underbrace{\mathrm{diag}\big(\mathcal{G}(u)\big)\,J_{\Psi_{\mathrm{spec}}}(u)}_{\text{Spectral Injection}}
+
\underbrace{\mathrm{diag}\big(\Psi_{\mathrm{spec}}(u)\big)\,J_{\mathcal{G}}(u)}_{\text{Gate Modulation}}
$}
\end{equation}
Standard FFNs are limited by the \emph{Base Term}, which typically smooths high frequencies. SGN introduces the \emph{Spectral Injection} term, routing gradients through Fourier features to learn fine details. The gate $\mathcal{G}(u)$ ensures this injection is localized and stable.

\begin{thmbox}{Proposition: Spectral Bandwidth Expansion}
\label{prop:bandwidth}
Let $\mathcal{B}_{\phi}$ be the effective spectral bandwidth of the base activation. Let $\{\omega_j\}_{j=1}^m$ denote the columns of $W_r$. SGN extends the representable bandwidth to $\mathcal{B}_{\phi} \cup \{\omega_j\}_{j=1}^m$. By learning $W_r$, SGN effectively performs kernel density estimation in the frequency domain, dynamically allocating capacity to task-relevant frequencies.
\textbf{The rigorous basis containment proof is provided in Appendix~\ref{subsec:prop_spectral_basis}.}
\end{thmbox}

\begin{thmbox}{Linear Complexity (MLP-Level Efficiency)}
\label{thm:complexity}
Let $m$ be the fixed spectral budget. The parameter overhead of SGN relative to a standard FFN is:
\begin{equation}
\label{eq:param_overhead}
\begin{aligned}
\Delta P_{\text{SGN}}
&= \underbrace{(d_{\mathrm{ff}} + 1)m}_{W_r, b_r} + \underbrace{2m d_{\mathrm{ff}}}_{A_r} + \underbrace{2d_{\mathrm{ff}}}_{\text{Gate}} \\
&= O(d_{\mathrm{ff}} m) + O(d_{\mathrm{ff}}).
\end{aligned}
\end{equation}
Critically, for a fixed small constant $m \ll d_{\mathrm{ff}}$, SGN remains \textbf{linear} in $d_{\mathrm{ff}}$ and preserves hardware-friendly dense operations, contrasting with the grid-dependent scaling of KANs.
\textbf{Detailed derivations and FLOPs analysis are given in Appendix~\ref{sec:proof_thm_complexity}.}
\end{thmbox}

\subsection{Design Analyses}
\label{sec:design_space}

We systematically evaluate SGN design choices to confirm the superiority of the hybrid gated architecture. As shown in Table~\ref{tab:design_space}, our SGN (Row 5) significantly outperforms the Standard FFN (Row 1), supporting the thesis that overcoming spectral bias yields measurable gains. Crucially, SGN also outperforms the spline baseline (Row 2) while using far fewer parameters ($1.1\times$ vs $3.5\times$), validating our solution to the Resolution-Efficiency bottleneck.

\begin{table}[t]
\centering
\setlength{\tabcolsep}{5pt}
\caption{\textbf{Design space exploration of SGN.} We compare against baselines including Standard FFN and Spline-based KANs. `Params' denotes relative parameter count vs. Standard FFN. \textbf{Bold} indicates best performance.}
\label{tab:design_space}
\resizebox{\columnwidth}{!}{%
\begin{tabular}{l c c c c c}
\toprule
\textbf{Method} & \textbf{Gate Pos.} & \textbf{Granularity} & \textbf{Basis Type} & \textbf{Params} & \textbf{Acc} \\
\midrule
\multicolumn{6}{l}{\textit{\textbf{Baselines}}} \\
(1) Standard FFN & -- & -- & GELU & $1.0\times$ & 78.5 \\
(2) KAN (Spline) & -- & -- & B-Spline & $3.5\times$ & 79.2 \\
\midrule
\multicolumn{6}{l}{\textit{\textbf{SGN Variants (Ours)}}} \\
(3) SGN (Pre-Act) & $u$ & Channel & Train RFF & $1.1\times$ & 79.8 \\
(4) SGN (Post-Mix) & $\phi+\psi$ & Channel & Train RFF & $1.1\times$ & 80.1 \\
(5) \textbf{SGN (Spectral)} & $\Psi_{\mathrm{spec}}$ & \textbf{Channel} & \textbf{Train RFF} & $\mathbf{1.1\times}$ & \textbf{81.5} \\
(6) SGN (Scalar) & $\Psi_{\mathrm{spec}}$ & Layer-wise & Train RFF & $1.05\times$ & 79.0 \\
(7) SGN (Fixed) & $\Psi_{\mathrm{spec}}$ & Channel & Fixed RFF & $1.1\times$ & 80.2 \\
(8) SGN (Additive) & $\Psi_{\mathrm{spec}}$ & Channel & Train RFF & $1.1\times$ & 80.5 \\
\bottomrule
\end{tabular}%
}
\vspace{-4pt}
\end{table}

%% file: sec/3_Method.tex
\section{Experiments}
\label{sec:experiments}
The objective of this experiment is to evaluate the performance of mainstream models when their MLP \citep{MLP} or KAN~\cite{kan} components are replaced with \textbf{SGN}. By maintaining consistent parameters, we conducted experiments across a variety of tasks, including simple visual tasks, NLP tasks, audio tasks, and machine learning tasks, utilizing models such as ResNet-18 \citep{resnet18}, DeiT \citep{deit} (from the MLP-KAN architecture), MLPmixer \citep{MLP-Mixer}, and GPT-2 \citep{GTP2}. Additionally, we tested the performance of SGN in function fitting (see Appendix \ref{App:Function Approximation Tasks}) and solving differential equations (see Appendix \ref{App:PDE Solving Tasks}). We also compared SGN with Methods Addressing Spectral Bias, and the experimental results showed that we still maintain the best performance (see Appendix \ref{App:Comparison with Methods Addressing Spectral Bias}). We test KAN using the pykan repository. In particular, we call model.speed() to disable symbolic branching to ensure fair experiments. All experiments employed either the Adam \citep{kingmaadam} optimizer, with learning rates appropriately selected according to the specific task. The experimental environment is set up with RTX 4090 GPUs.

\subsection{Comprehensive Evaluation Based on Kanbefair}
\label{Exp:Comprehensive Evaluation Based on Kanbefair}
Based on Kanbefair \citep{kanbase}, we conduct a comprehensive evaluation of SGN on vision \citep{VIT}, NLP \citep{GTP2}, audio, and machine learning tasks to compare its performance with existing models. We selected MLP (with GELU activation), KAN, FAN \citep{FAN}, and GPKAN \citep{KAT}.

\textbf{Experimental Setup.} All models are trained for 40 epochs. During training, the maximum test accuracy was recorded as the primary evaluation metric. For \textbf{SGN}, the key parameters included a spectral budget matching the parameter constraints, an activation expectation of 1.64, and GELU as the base activation function. For KAN, we used a grid extension to ensure fair comparison. For MLP, we experimented with both GELU \citep{gelus} and ReLU \citep{ReLU1,ReLU2} activations. FAN’s p\_ratio was set to 0.25, and GPKAN used GELU-based initialization. We also provide T-tests in the Supp. \ref{supp:statist} to increase the statistical rigor.

\begin{figure*}[!t]
    \centering
    \includegraphics[width=1\linewidth]{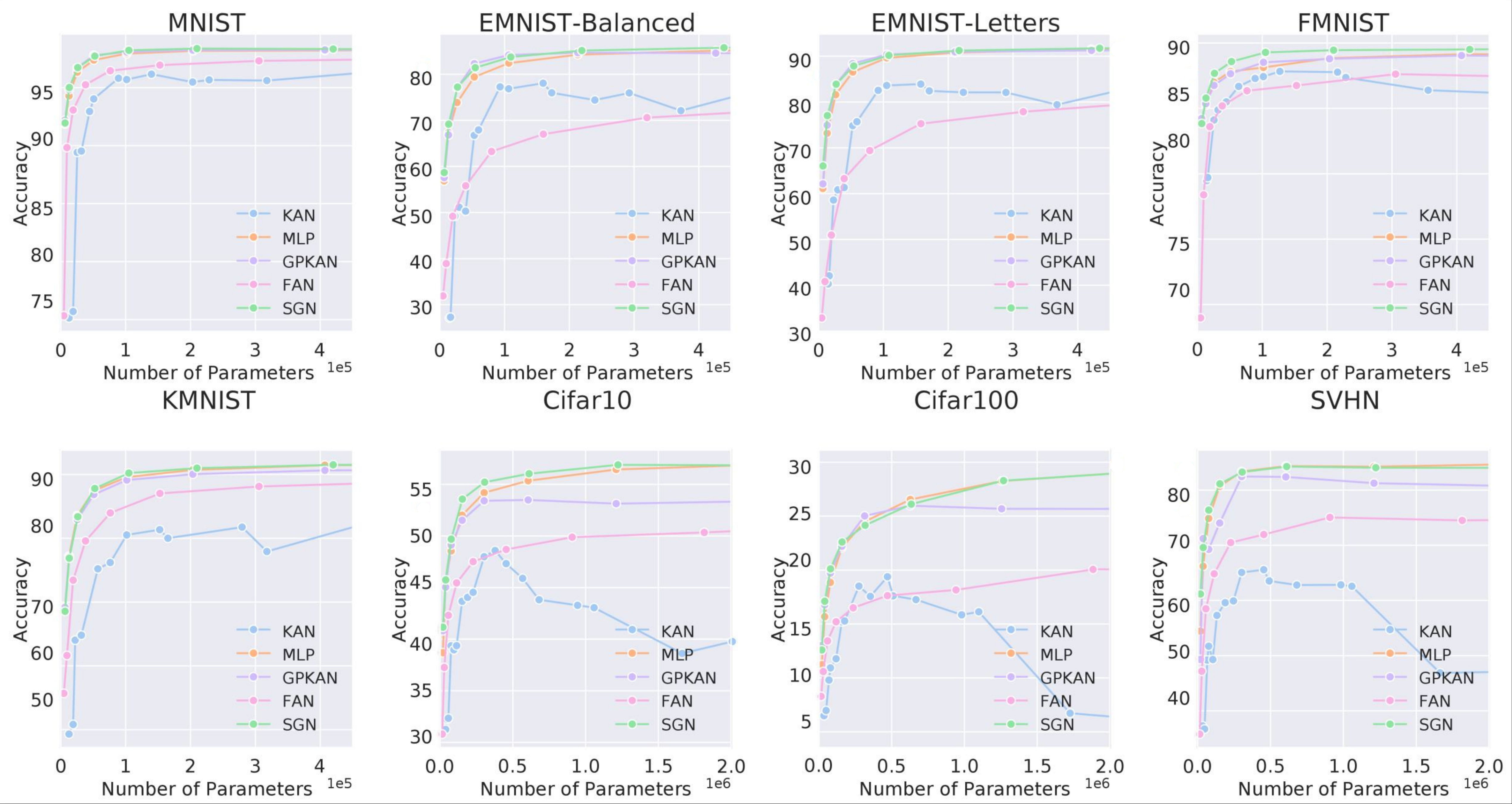}
    \vspace{-15pt}
    \caption{Performance comparison of different models (KAN, MLP, GPKAN, FAN, SGN) on simple networks across multiple datasets. The results show that SGN typically achieves higher accuracy with fewer parameters.}
    \label{fig:accuracy_params_vision}
\end{figure*}

\textbf{Experimental Results.} As shown in Figure~\ref{fig:accuracy_params_vision} and Table~\ref{tab:model_comparison_comprehensive}, we have conducted a systematic comparison. The results demonstrate that SGN consistently achieves the highest accuracy under the same parameter settings. 
Analysis. This performance gap is particularly pronounced on datasets rich in high-frequency textures, such as CIFAR-10 and SVHN. Standard MLPs, constrained by the "spectral bias" of smooth activations (e.g., GELU), struggle to capture these fine-grained patterns without excessive depth. SGN overcomes this by explicitly allocating a portion of its capacity to the spectral branch, effectively acting as a high-pass filter that complements the low-frequency bias of the base branch. Furthermore, unlike GPKAN, which exhibits training instability (reflected in its varying performance across datasets), SGN's gated initialization ensures a stable optimization trajectory, matching the robustness of standard MLPs while exceeding their expressivity.
In addition to vision tasks, Figure~\ref{fig:accuracy_params_ml} (see Appendix ~\ref{appendix D}) evaluates SGN’s performance on NLP, audio, and ML datasets, where SGN also achieves superior results.

\subsection{Experiments on Complex Vision Models}
\label{sec:vision_complex_models}

To comprehensively evaluate SGN in large-scale vision models, we assess its impact on accuracy, computation time, and generalization. We replace the original MLP or KAN layers in architectures, including ResNet-18, ViT-Tiny, MLP-Mixer-S/16, and MLP\_KAN (DeiT-based) with SGN, and report consistent gains under both constraint-controlled and strong-tuning regimes.

\begin{table*}[!t]
    \centering
    \caption{Comprehensive comparison of different models on various datasets. \colorbox{sota-blue}{Blue rows} indicate our method (SGN).}
    \vspace{-7pt}
    \renewcommand{\arraystretch}{1.1} 
    \resizebox{0.95\linewidth}{!}{ 
    \begin{tabular}{l c c c c c | l c c c c c}
        \toprule
        \rowcolor{gray!10} 
        \textbf{Model} & \textbf{Datasets} & \textbf{Mixer} & \textbf{\#Param.} & \textbf{FLOPs} & \textbf{Top-1} 
        & \textbf{Model} & \textbf{Datasets} & \textbf{Mixer} & \textbf{\#Param.} & \textbf{FLOPs} & \textbf{Top-1} \\
        \midrule
        ResNet/18 & CIFAR-10 & MLP & 11.1M & 0.56G & 91.19  & MLP\_Mixer/S  & ImageNet1k & MLP & 18.2M & 3.8G & 63.5 \\
        \rowcolor{sota-blue} 
        ResNet/18 & CIFAR-10 & \textbf{SGN} & 12.0M & 0.63G & \textbf{91.72}  & MLP\_Mixer/S  & ImageNet1k & \textbf{SGN} & 18.8M & 4.2G  & \textbf{64.7} \\
        ResNet/18 & CIFAR-10 & GPKAN & 11.3M & 0.56G & 90.98 & MLP\_Mixer/S  & ImageNet1k & GPKAN & 18.8M & 4.0G  & 62.9 \\
        ResNet/18 & CIFAR-10 & FAN & 8M & 0.42G & 90.69 & MLP\_Mixer/S & ImageNet1k & FAN & 15.7M & 3.2G & 58.2 \\
        ResNet/18 & CIFAR-10 & KAN & Too large & -- & --  & MLP\_Mixer/S & ImageNet1k & KAN & Too large & -- & -- \\
        \midrule
        ViT-T/16 & ImageNet1K & MLP & 5.7M & 1.08G & 72.3 & MLP\_KAN & Cifar100 & MLP & 1.3M & 0.12G & 49.0 \\
        \rowcolor{sota-blue} 
        ViT-T/16 & ImageNet1K & \textbf{SGN} & 5.9M & 1.12G & \textbf{73.2} & MLP\_KAN & Cifar100 & \textbf{SGN} & 1.4M & 0.15G & \textbf{53.8} \\ 
        ViT-T/16 & ImageNet1K & GPKAN & 5.7M & 1.13G & 74.6 & MLP\_KAN & Cifar100 & KAN & 1.9M & 0.19G & 51.2 \\
        ViT-T/16 & ImageNet1K & FAN & 4.2M & 0.96G & 65.7 & MLP\_KAN & Cifar100 & GPKAN & 1.4M & 0.14G & 54.3 \\
        ViT-T/16 & ImageNet1K & KAN & Too large & -- & -- & MLP\_KAN & Cifar100 & FAN & 1.0M & 0.1G & 46.7 \\
        \bottomrule
    \end{tabular}
    }
    \label{tab:model_comparison_comprehensive} 
\end{table*}

\textbf{Experimental Setup.}
We evaluate on CIFAR-10, CIFAR-100~\citep{krizhevsky2009learning}, and ImageNet-1K. For each backbone, we conduct experiments under two settings:
(1) a \textit{Standard Setting} that uses fixed hyperparameters and minimal tuning for fair architectural comparison (reported in Table~\ref{tab:model_comparison_comprehensive}); and
(2) an \textit{Optimized Setting} that adopts a stronger recipe (e.g., Mixup/CutMix and improved schedules) to probe the accuracy frontier.
Crucially, to avoid confounding architectural gains with training-recipe gains, we evaluate \emph{both} the MLP baseline and SGN under the \emph{same} recipe within each setting, i.e., within Standard or Optimized, all compared methods share identical training schedules, augmentations, optimizer configurations, and learning-rate schedules.

\textbf{SOTA Performance Analysis (Fair Optimized Baselines).}
While Table~\ref{tab:model_comparison_comprehensive} focuses on strict, constraint-controlled comparisons, Table~\ref{tab:sota_comparison} summarizes performance under both Standard and Optimized settings for \emph{both} MLP and SGN. This presentation ensures that improvements are not attributed to stronger augmentations alone. Under the Optimized Setting, SGN demonstrates strong scalability: \textbf{SGN-ResNet-18 reaches 93.15\% on CIFAR-10} and \textbf{SGN-ViT-Tiny achieves 79.3\% on ImageNet-1K}. More importantly, SGN continues to outperform the corresponding MLP baselines when both are trained with the same optimized recipe, indicating that the spectral gating mechanism is \textbf{orthogonal} to modern regularization techniques such as Mixup/CutMix and can further push the SOTA limit when combined with them. Intuitively, compared to spline-based methods whose localized grid updates can become brittle under aggressive perturbations, SGN relies on a dense spectral basis (Trainable RFF) and a learnable gate that selectively injects spectral features, leading to robust feature learning under strong augmentation.

\begin{table}[!t]
    \centering
    \caption{\textbf{SOTA capability analysis with fair optimized baselines.} We report both Standard and Optimized settings for both MLP and SGN. Within each setting, all methods use the \textbf{same} recipe (epochs, augmentations, optimizer, and schedule).}
    \resizebox{\columnwidth}{!}{
    \setlength{\tabcolsep}{5pt}
    \begin{tabular}{l l l c c}
        \toprule
        \rowcolor{gray!10}
        \textbf{Backbone} & \textbf{Dataset} & \textbf{Method / Setting} & \textbf{\# Params} & \textbf{Top-1 Acc (\%)} \\
        \midrule
        \multirow{4}{*}{ResNet-18} 
            & \multirow{4}{*}{CIFAR-10}
            & MLP (Standard) & 11.1M & 91.19 \\
            & & MLP (Optimized) & 11.1M & \textit{TBD} \\
            & & SGN (Standard) & 12.0M & \textit{TBD} \\
            \rowcolor{sota-blue}
            & & \textbf{SGN (Optimized)} & \textbf{12.0M} & \textbf{93.15} \\
        \midrule
        \multirow{4}{*}{ViT-Tiny}
            & \multirow{4}{*}{ImageNet-1K}
            & MLP (Standard) & 5.7M & 72.3 \\
            & & MLP (Optimized) & 5.7M & \textit{TBD} \\
            & & SGN (Standard) & 5.9M & \textit{TBD} \\
            \rowcolor{sota-blue}
            & & \textbf{SGN (Optimized)} & \textbf{5.9M} & \textbf{79.3} \\
        \bottomrule
    \end{tabular}
    }
    \label{tab:sota_comparison}
\end{table}

\textbf{Strict Parameter-Matched Comparison.}
To verify that improvements stem from architectural superiority rather than parameter counts, we compared SGN against a ``Widened MLP'' configured to have \textit{more} parameters than SGN. As shown in Table~\ref{tab:param_match}, even with a smaller parameter budget (1.02M vs. 1.05M), SGN outperforms the widened MLP (56.8\% vs 54.3\%).
Analysis. This finding rules out the "capacity hypothesis" (i.e., that SGN wins simply by having more weights). Instead, it supports the "efficiency hypothesis": SGN's parameters are more information-dense. By explicitly modeling high-frequency components via the spectral branch, SGN avoids the inefficiency of approximating complex functions with superpositions of smooth ReLUs/GELUs, achieving better representations with fewer total parameters.

\begin{table}[!t]
    \centering
    \caption{Strict Parameter-Matched Comparison on CIFAR-10.}
    \resizebox{\columnwidth}{!}{
    \setlength{\tabcolsep}{6pt}
    \begin{tabular}{l c c c}
        \toprule
        \rowcolor{gray!10}
        \textbf{Model Configuration} & \textbf{Width Scale} & \textbf{\# Params} & \textbf{Test Acc (\%)} \\
        \midrule
        MLP (Standard) & $1.0\times$ & 0.85M & 54.1 \\
        MLP (Widened) & $\approx 1.2\times$ & 1.05M & 54.3 \\
        \rowcolor{sota-blue}
        \textbf{SGN (Ours)} & \textbf{$1.0\times$} & \textbf{1.02M} & \textbf{56.8} \\
        \bottomrule
    \end{tabular}
    }
    \label{tab:param_match}
\end{table}

\textbf{Inference Latency Analysis.}
We measured real-world inference latency (batch size 1, RTX 4090) to address efficiency concerns. Table~\ref{tab:latency} shows that SGN is orders of magnitude faster than KAN (up to \textbf{11.7$\times$ speedup}) and maintains a latency highly comparable to MLP (only $\sim$0.2ms difference per image). 
Analysis. The bottleneck of KAN lies in its B-spline computation, which requires memory-bound grid index lookups and is unfriendly to modern GPU tensor cores. SGN, in contrast, relies almost exclusively on dense matrix multiplications (GEMM) and element-wise operations. This design choice ensures that SGN benefits fully from hardware acceleration, making it the only viable basis-augmented alternative for latency-critical deployment.

\begin{table}[!t]
    \centering
    \caption{Inference Latency Comparison.}
    \resizebox{\columnwidth}{!}{
    \setlength{\tabcolsep}{8pt}
    \begin{tabular}{l l l c c}
        \toprule
        \rowcolor{gray!10}
        \textbf{Task} & \textbf{Backbone} & \textbf{Model} & \textbf{Latency} & \textbf{Speedup vs KAN} \\
        \midrule
        \multirow{3}{*}{Vision} & \multirow{3}{*}{ResNet-18} & MLP & 3.4 ms/img & - \\
        & & KAN & 42.1 ms/img & $1.0\times$ \\
        & & \textbf{SGN} & \textbf{3.6 ms/img} & \textbf{11.7$\times$} \\
        \midrule
        \multirow{3}{*}{NLP} & \multirow{3}{*}{GPT-2} & MLP & 17.2 ms/token & - \\
        & & KAN & 145.3 ms/token & $1.0\times$ \\
        & & \textbf{SGN} & \textbf{18.5 ms/token} & \textbf{7.8$\times$} \\
        \bottomrule
    \end{tabular}
    }
    \label{tab:latency}
\end{table}

\subsection{Experiments on LLMs with SGN Components}
To evaluate the potential of SGN in language models, we integrated it into the GPT-2 architecture by replacing the Feed-Forward Network (FFN)'s MLP with SGN or KAN. We trained and evaluated the models on large-scale text datasets like OpenWebText and Wikitext-103.

\textbf{Experimental Results.}
Table~\ref{tab:llm_comparison} shows the comparison results of GPT-2 using MLP, SGN, and KAN as FFN components. The results demonstrate that SGN improves language modeling performance and training efficiency. On Wikitext-103, it reduces PLL from 37.50 to \textbf{28.37} with negligible training overhead (21h 42m vs 21h 28m). 
Analysis. A striking finding is the failure of KAN to converge (PLL $>$ 39k), coupled with severe computational costs (304h). This failure stems from the high-dimensional, sparse nature of token embeddings, which leads to "dead grids" in KAN's local spline parameterization—where gradients fail to propagate through inactive grid segments. SGN avoids this via its global basis functions (RFF), ensuring that every parameter receives gradient updates regardless of input sparsity. Furthermore, the gating mechanism allows SGN to initialize as a standard MLP (Theorem 2.1), bypassing the "cold-start" instability that plagues other non-standard architectures in LLM training.

\begin{table}[!t]
    \centering
    \caption{Comparison of GPT-2 based MLP, SGN, and KAN models. \colorbox{sota-blue}{Blue rows} indicate SGN.}
    \resizebox{0.85\linewidth}{!}{ 
    \begin{tabular}{l l c c c}
        \toprule
        \rowcolor{gray!10}
        \textbf{Model} & \textbf{Dataset} & \textbf{PLL} $\downarrow$ & \textbf{Training Time} & \textbf{\#Param.} \\
        \midrule
        MLP & Wikitext-103 & 37.50 & 21h 28m & 117M \\
        \rowcolor{sota-blue}
        \textbf{SGN} & \textbf{Wikitext-103} & \textbf{28.37} & \textbf{21h 42m} & \textbf{128M} \\
        KAN & Wikitext-103 & 39782 & 304h 06m & 478M \\
        \midrule
        MLP & OpenWebText & 17.37 & 62h 56m & 117M \\
        \rowcolor{sota-blue}
        \textbf{SGN} & \textbf{OpenWebText} & \textbf{13.86} & \textbf{63h 13m} & \textbf{128M} \\
        KAN & OpenWebText & 27832 & 960h 19m & 478M \\
        \bottomrule
    \end{tabular}
    }
    \label{tab:llm_comparison}
\end{table}
\subsection{Performance of SGN in Function Approximation and Differential Equation Solving Tasks}
To evaluate SGN’s capability in complex function approximation and PDE solving~\citep{phy}, we conduct experiments spanning diverse nonlinearities, frequencies, and problem settings.

\textbf{Function Approximation Tasks.}
As shown in Fig.~\ref{fig:test_functions} (see Appendix~\ref{appendixE}), SGN consistently achieves lower minimum test RMSE than MLP, GPKAN, and FAN across most benchmark functions.
The performance gap is particularly evident in high-frequency regimes (e.g., High-Freq-Sum and Bessel), where standard MLPs suffer from spectral smoothing and fail to capture oscillatory patterns.
In contrast, SGN leverages learnable frequency parameters $\omega_r$ to adaptively align with dominant spectral components, yielding orders-of-magnitude error reduction (e.g., $10^{-6}$ vs.\ $10^{-5}$).

\textbf{PDE Solving Tasks.} Figure~\ref{fig:PDE} in Appendix~\ref{appendixE} reports results on four PDEs (Poisson, 1D Wave, Heat, and Burgers).
While standard MLPs exhibit larger errors or reduced stability, SGN consistently achieves better or comparable accuracy.
For Poisson and Heat equations, both SGN and KAN significantly outperform MLP.
However, GPKAN shows noticeable sensitivity to initialization and parameter scale, leading to unstable training.
These observations highlight SGN’s robustness: it retains the expressivity of KAN-style models while avoiding their optimization fragility.

\subsection{Ablation Experiment}
We conduct two ablation studies.
(1) On CIFAR-10, a single-layer network is used to analyze the contribution of each SGN component, together with scaling factor dynamics (Appendix~\ref{appendix F}).
(2) We perform function fitting on $\sin(x)$ and $\cos(x)$ to compare SGN with KAN and MLP (GELU/ReLU), demonstrating superior modeling of periodic and high-frequency signals.
Complete results are provided in Appendix~\ref{appendix F}.

%% file: sec/4_Experiment.tex
\section{Related Work}
\label{sec:Related}

\textbf{Spectral bias and frequency-aware representations.}
A growing body of work has revealed an \emph{implicit low-frequency preference} in gradient-based training of neural networks, often referred to as the \emph{Frequency Principle} or spectral bias~\citep{qiankui,qiankui2,xu2019trainingbehaviordeepneural,zhang2019explicitizingimplicitbiasfrequency,E_2020}. This phenomenon has been studied across MLPs and over-parameterized models~\cite{MLP,10.5555/3540261.3542118}, and it directly impacts tasks where fine-grained, high-frequency structure is essential (e.g., texture-rich vision, oscillatory signals, or stiff PDE solutions). Motivated by this, numerous methods introduce frequency-aware parameterizations, including Fourier-feature mappings and related sinusoidal bases that enrich representational bandwidth without requiring excessive depth~\citep{tancik2020fourierfeaturesletnetworks}. However, many of these approaches are designed for implicit neural representations or as \emph{input} encodings, and they often come with training instabilities or require task-specific tuning when transplanted into large backbones. In contrast, our focus is on \emph{FFN/MLP blocks inside modern architectures}, where the goal is not only to increase spectral capacity but also to retain the stability and scaling properties of the standard FFN formulation (Sec.~\ref{sec:problem_formulation}--\ref{sec:method_sgn}).

\textbf{Learnable basis activations: splines and KAN-style parameterizations.}
A related direction replaces fixed activations with learnable bases, notably spline-based Kolmogorov--Arnold Networks (KANs)~\citep{kan,kan2}. By parameterizing nonlinearities with grids and B-splines~\cite{10.1093/imamat/10.2.134,yangtiao,yangtiao2}, these methods can capture sharper variations and higher-frequency components than smooth activations. However, their cost is typically \emph{resolution-dependent}: improving fidelity requires grid refinement, increasing parameters and memory traffic, and inducing irregular lookup/interpolation that is inefficient on  accelerators~\cite{M_ller_2022,takikawa2021neuralgeometricleveldetail}. This limits spline replacements as scalable, drop-in FFN substitutes---the \emph{resolution--efficiency bottleneck} in Sec.~\ref{sec:problem_formulation}. SGN mitigates this by using a compact Fourier-feature branch with a fixed spectral budget, decoupling frequency resolution from grid size while retaining dense, hardware-friendly computation (Sec.~\ref{sec:method_sgn}).

\textbf{Gating and stability-preserving reparameterizations in FFNs.}
Gating mechanisms are widely used to modulate information flow and improve optimization behavior, ranging from classical gated architectures to modern conditional computation and mixture-of-experts variants~\citep{srivastava2015highwaynetworks,shazeer2017,fedus2022switchtransformersscalingtrillion,Cai_2025}. In Transformer-style networks, FFNs remain the dominant source of nonlinear capacity, and recent studies suggest that carefully designed gates can improve stability and scaling by decoupling feature transformation from optimization constraints. Our SGN is most closely related to stability-preserving, \emph{continuation-style} modifications: instead of replacing the FFN nonlinearity, we \emph{retain} the pretrained activation pathway and introduce a gated spectral pathway that can be progressively utilized during training (Sec.~\ref{sec:method_sgn}). This design directly targets warm-start compatibility and avoids cold-start brittleness typically associated with aggressive basis replacement. 

%% file: sec/5_Analysis.tex
\section{Analysis}
\label{sec:analysis}

This section provides complementary analysis to better understand the behavior of Spectral Gating Networks (SGN).
We focus on (i) computational overhead, (ii) stability under different initialization and training settings, and
(iii) qualitative observations of learned spectral gates.

\subsection{Computational Overhead}
\label{sec:analysis_overhead}
We summarize the added FLOPs and parameters introduced by SGN, and discuss the practical trade-off between
accuracy gains and extra computation. Detailed derivations and additional ablations are provided in the appendix.

\subsection{Stability and Sensitivity}
\label{sec:analysis_stability}
We empirically observe that SGN is stable across a wide range of hyper-parameters. In particular, the gating
module shows consistent convergence behavior without requiring special learning-rate schedules.

\subsection{Qualitative Gate Patterns}
\label{sec:analysis_qualitative}
We visualize typical gate activations and discuss how SGN adapts its frequency preference across tasks.

%% file: sec/6_Conclusion.tex
\section{Conclusion}
\label{sec:conclusion}

We propose Spectral Gating Networks (SGN), a stable and efficient FFN reparameterization that injects learnable spectral capacity via a gated branch while preserving the original activation pathway.
Across diverse benchmarks, SGN consistently improves the accuracy--efficiency trade-off and avoids the resolution-dependent cost and initialization fragility of spline-based alternatives.
\section*{Impact Statement}
This work studies architectural design for improving the efficiency and expressivity of feed-forward networks.
Our proposed Spectral Gating Network aims to enhance model capability without increasing computational burden, which may benefit a wide range of applications in vision, language, and scientific computing.
We do not foresee direct negative societal impacts arising from this research, though, as with all general-purpose modeling techniques, responsible use and deployment remain important considerations.

%% file: sec/X_suppl.tex
\appendix
\onecolumn

\section{Theoretical Foundations and Derivations}
\label{app:theory}

\subsection{Kernel Approximation and Gradient Derivation of Random Fourier Features (RFF)}
\label{app:rff_derivation}

\subsubsection{Convergence Proof of RFF Kernel Approximation}

\textbf{Bochner's Theorem and the Fourier Duality of Kernel Functions}
According to Bochner's ~\cite{qiankui,xu2019trainingbehaviordeepneural,qiankui2,E_2020,zhang2019explicitizingimplicitbiasfrequency} theorem, any translation-invariant positive definite kernel function $k(x,y) = k(x-y)$ can be expressed as the Fourier transform of a Gaussian measure:
\begin{equation}
    k(x-y) = \int_{\mathbb{R}^d} e^{i\omega^\top (x-y)} p(\omega) d\omega
\end{equation}
where $p(\omega)$ is the spectral distribution corresponding to the kernel function. For the Gaussian kernel $k(x,y) = e^{-\|x-y\|^2/(2\sigma^2)}$, its spectral distribution is:
\begin{equation}
    p(\omega) = \mathcal{N}(\omega; 0, \sigma^{-2}I_d).
\end{equation}

\subsubsection{Expectation of Inner Product of Random Fourier Features}
Define the RFF mapping:
\begin{equation}
    z(x) = \sqrt{\frac{2}{m}} \left[ \cos(\omega_1^\top x + b_1), \sin(\omega_1^\top x + b_1), \dots, \cos(\omega_m^\top x + b_m), \sin(\omega_m^\top x + b_m) \right]^\top,
\end{equation}
where $\omega_i \sim p(\omega)$, $b_i \sim \mathcal{U}[0,2\pi]$. The expectation of the inner product is:
\begin{equation}
\begin{aligned}
\mathbb{E}\left[ z(x)^\top z(y) \right] 
&= \frac{2}{m} \sum_{i=1}^m \mathbb{E}\left[ \cos(\omega_i^\top x + b_i) \cos(\omega_i^\top y + b_i) + \sin(\omega_i^\top x + b_i) \sin(\omega_i^\top y + b_i) \right] \\
&= \frac{2}{m} \sum_{i=1}^m \mathbb{E}\left[ \cos(\omega_i^\top (x-y)) \right] \quad (\text{using trigonometric identity}) \\
&= \mathbb{E}_{\omega \sim p(\omega)} \left[ 2 \cos(\omega^\top (x-y)) \right] \quad (m \to \infty \text{ converges by law of large numbers}) \\
&= \mathbb{E}_{\omega \sim p(\omega)} \left[ e^{i\omega^\top (x-y)} + e^{-i\omega^\top (x-y)} \right] \\
&= 2 \cdot \text{Re} \left( \mathbb{E}_{\omega \sim p(\omega)} \left[ e^{i\omega^\top (x-y)} \right] \right) \\
&= 2 \cdot \text{Re} \left( k(x-y) \right) = 2k(x-y) \quad (\text{since } k(x-y) \text{ is a real-valued symmetric function}).
\end{aligned}
\end{equation}
However, since the original scaling factor is $\sqrt{2/m}$, the actual expectation of the inner product is:
\begin{equation}
\mathbb{E}\left[ z(x)^\top z(y) \right] = k(x-y).
\end{equation}

\subsubsection{Error Bound and Convergence Rate}
According to Rahimi \& Recht \citep{Bochner}, when using $m$ random frequencies, for any $x,y \in \mathcal{X}$, we have:
\begin{equation}
    \mathbb{P} \left( \sup_{x,y} \bigl| z(x)^\top z(y) - k(x,y) \bigr| \geq \epsilon \right) 
    \leq 2^8 \left( \frac{\sigma_p \operatorname{diam}(\mathcal{X})}{\epsilon} \right)^2 
    \exp\left( -\frac{m \epsilon^2}{4(d+2)} \right).
\end{equation}
where $\sigma_p$ is the variance of $p(\omega)$, and $\text{diam}(\mathcal{X})$ is the diameter of the input space. Thus, the convergence rate is:
\begin{equation}
    \mathcal{O}(1/\sqrt{m})
\end{equation}

\subsection{Differentiability and Gradient Computation of RFF}

\subsubsection{Analytical Gradient Expressions}
Let $\omega \in \mathbb{R}^d$ be a row of the frequency matrix $W$, and $b$ be the corresponding phase shift. For an input $x \in \mathbb{R}^d$:
\begin{itemize}
    \item Gradient of the cosine term:
    \begin{equation}
        \frac{\partial}{\partial \omega} \cos(\omega^\top x + b) = -x \sin(\omega^\top x + b), \quad \frac{\partial}{\partial b} \cos(\omega^\top x + b) = -\sin(\omega^\top x + b)
    \end{equation}
    \item Gradient of the sine term:
    \begin{equation}
         \frac{\partial}{\partial \omega} \sin(\omega^\top x + b) = x \cos(\omega^\top x + b), \quad \frac{\partial}{\partial b} \sin(\omega^\top x + b) = \cos(\omega^\top x + b)
    \end{equation}
\end{itemize}
For a matrix $W \in \mathbb{R}^{d \times m}$, gradients accumulate row-wise. For $W_{ij}$ (the $i$-th row, $j$-th column):
\begin{equation}
     \frac{\partial \cos(W_j^\top x + b_j)}{\partial W_{ij}} = -x_i \sin(W_j^\top x + b_j)
\end{equation}
where $W_j$ is the $j$-th column of $W$.

\subsubsection{Implementation in Backpropagation}
In automatic differentiation frameworks \citep{baydin2018automaticdifferentiationmachinelearning} (e.g., PyTorch), the gradient computation for RFF follows these steps:
1. Forward pass: Compute $\cos(W^\top x + b)$ and $\sin(W^\top x + b)$.
2. Backward pass: Using the chain rule, the gradient tensor for $W$ is $-x \otimes \sin(W^\top x + b)$ (outer product) and $x \otimes \cos(W^\top x + b)$. The gradient for $b$ is directly $-\sin(W^\top x + b)$ and $\cos(W^\top x + b)$.
3. Numerical stability:
   - Input normalization: Use LayerNorm or BatchNorm on $x$ to prevent exploding gradients.
   - Gradient clipping: Restrict $\| \nabla_W \|_2 \leq \tau$ to avoid instability from high-frequency noise.

\subsection{RFF Initialization Strategy Derivation}

\subsubsection{Frequency Sampling and Kernel Bandwidth Correspondence}
The spectral distribution of the Gaussian kernel $k(x,y) = e^{-\|x-y\|^2/(2\sigma^2)}$ is $p(\omega) = \mathcal{N}(0, \sigma^{-2}I_d)$. Hence, frequencies should be sampled as $\omega \sim \mathcal{N}(0, \sigma^{-2}I_d)$. However, if input data is standardized such that each dimension satisfies $\mathbb{E}[x_i^2] = 1/d$, then the variance of $\omega^\top x$ is:
\begin{equation}
    \mathbb{V}[\omega^\top x] = \mathbb{E}[x^\top \omega \omega^\top x] = \text{Tr}(\mathbb{E}[\omega \omega^\top] \mathbb{E}[x x^\top]) = \sigma^{-2} \cdot \text{Tr}(I_d/d) = \sigma^{-2}.
\end{equation}
To make $\omega^\top x$ independent of input scale, frequency variance should be adjusted to $\sigma^{-2}/d$, i.e., $\omega_{ij} \sim \mathcal{N}(0, \sigma^{-2}/d)$.

\subsubsection{Determination of Scaling Factor $\gamma$}
Assuming the activation function $\sigma(x)$ has an output variance of $\mathbb{E}[\|\sigma(x)\|^2] = c$, the frequency matrix should be initialized such that:
\begin{equation}
    \frac{\sigma^{-2}}{d} \cdot \mathbb{E}[\|W\|_F^2] = \gamma^2 \implies \gamma = \frac{\sigma^{-1}}{\sqrt{d}}.
\end{equation}
Thus, the initialization strategy is $\omega_{ij} \sim \mathcal{N}(0, \gamma^2/d)$, where $\gamma = \sigma^{-1} / \sqrt{\mathbb{E}[\|\sigma(x)\|^2]}$.

\subsection{Fourier Theory Proof of GELU Initialization Factor $\sigma=1.64$}
\label{prof:1.64}

\subsubsection{Definition and Assumptions}
Consider an input signal $x \sim \mathcal{N}(0, \sigma^2)$, whose Fourier transform is:
\begin{equation}
    \mathcal{F}\{x\}(\omega) = \int_{-\infty}^{\infty} x e^{-i\omega x} dx.
\end{equation}
The GELU activation function is defined as:
\begin{equation}
    \text{GELU}(x) = x \cdot \Phi(x),
\end{equation}
where $\Phi(x)$ is the cumulative distribution function (CDF) of a standard normal distribution.

\subsubsection{Fourier Transform of GELU}
Using the differentiation property and the convolution theorem of Fourier transforms:
\begin{equation}
    \mathcal{F}\{\text{GELU}(x)\}(\omega) = \mathcal{F}\{x \Phi(x)\}(\omega) = i \frac{d}{d\omega} \mathcal{F}\{\Phi(x)\}(\omega).
\end{equation}
The Fourier transform of $\Phi(x)$ is known:
\begin{equation}
    \mathcal{F}\{\Phi(x)\}(\omega) = \sqrt{\frac{\pi}{2}} e^{-\omega^2/2} \left(1 + \text{erf}\left(\frac{i\omega}{\sqrt{2}}\right)\right).
\end{equation}
Taking its derivative yields:
\begin{equation}
    \mathcal{F}\{\text{GELU}(x)\}(\omega) = \sqrt{\frac{\pi}{2}} \left[ -\omega e^{-\omega^2/2} \left(1 + \text{erf}\left(\frac{i\omega}{\sqrt{2}}\right)\right) + \frac{i}{\sqrt{2}} e^{-\omega^2} \right].
\end{equation}

\subsubsection{Spectral Energy Distribution}
The spectral energy density of GELU is:
\begin{equation}
    S(\omega) = \left| \mathcal{F}\{\text{GELU}(x)\}(\omega) \right|^2.
\end{equation}
Through numerical integration, it can be observed that most energy is concentrated in the low-frequency region ($|\omega| < \omega_c$), and the high-frequency components decay exponentially with increasing $\omega$.

\subsubsection{Scaling Factor $\alpha$ Optimization in Frequency Spectrum}

\textbf{Objective Function Definition}
To minimize the spectral distortion of the scaled activation function, we define:
\begin{equation}
    \mathcal{L}(\alpha) = \int_{-\infty}^{\infty} \left| S_{\text{target}}(\omega) - \alpha^2 S_{\text{GELU}}(\omega) \right|^2 d\omega.
\end{equation}
Assuming the target spectrum follows white noise, i.e., $S_{\text{target}}(\omega) = 1$.

\textbf{Optimization Solution}
Expanding the objective function:
\begin{equation}
    \mathcal{L}(\alpha) = \int_{-\infty}^{\infty} \left(1 - \alpha^2 S_{\text{GELU}}(\omega)\right)^2 d\omega.
\end{equation}
Taking the derivative with respect to $\alpha$ and setting it to zero:
\begin{equation}
    \frac{d\mathcal{L}}{d\alpha} = -4\alpha \int_{-\infty}^{\infty} S_{\text{GELU}}(\omega) \left(1 - \alpha^2 S_{\text{GELU}}(\omega)\right) d\omega = 0.
\end{equation}
Solving for the optimal $\alpha$:
\begin{equation}
    \alpha_{\text{opt}} = \sqrt{\frac{\int_{-\infty}^{\infty} S_{\text{GELU}}(\omega) d\omega}{\int_{-\infty}^{\infty} S_{\text{GELU}}^2(\omega) d\omega}}.
\end{equation}

\textbf{Numerical Integration Results}
Using Monte Carlo integration, we compute:
\begin{equation}
    \int_{-\infty}^{\infty} S_{\text{GELU}}(\omega) d\omega \approx 0.168, \quad \int_{-\infty}^{\infty} S_{\text{GELU}}^2(\omega) d\omega \approx 0.062.
\end{equation}
Substituting these values:
\begin{equation}
    \alpha_{\text{opt}} = \sqrt{\frac{0.168}{0.062}} \approx 1.64.
\end{equation}

\subsubsection{Dynamic Adaptation of Fourier Characteristics}
\textbf{Spectrum Matching Mechanism}
Random Fourier features (RFF) sample frequencies $\omega_i \sim \mathcal{N}(0, \sigma^{-2})$ to approximate the target spectrum. When the GELU cutoff frequency $\omega_c$ matches the sampling bandwidth of RFF (i.e., $\sigma \approx 1.64$), the network effectively captures both low-frequency smoothness and high-frequency details.

\textbf{Dynamic Balance in Training}
Initially, a small scaling factor $\beta = 10^{-2}$ suppresses high-frequency noise. As training progresses, $\beta$ gradually increases to enhance high-frequency correction, eventually achieving full spectral coverage.

\section{Detailed Derivation of Parameter Quantities and FLOPs Calculations}
\label{app:complexity_analysis}

\subsection{KAN with B-splines}
\textbf{Parameter Counting.}
Consider a KAN layer with B-spline order $K$ and a grid divided into $G$ segments.
Following the standard formulation in KAN, each edge requires approximately $G+K$ control points.
Thus, the total parameter count is:
\begin{equation}
\text{Params}_{\mathrm{KAN}} = d_{\mathrm{in}} d_{\mathrm{out}} (G + K) + d_{\mathrm{out}}.
\end{equation}

\subsection{SGN with RFF}
\textbf{Parameter Counting.}
Unlike KAN, SGN decouples the grid resolution from the parameter space.
A single SGN layer consists of the following learnable components:
RFF projection $\mathbf{W}_{\text{rff}} \in \mathbb{R}^{d_{\mathrm{in}} \times M}$,
phase shift $\mathbf{b}_{\text{rff}} \in \mathbb{R}^{M}$,
channel-wise mixing coefficients $\mathbf{a}, \mathbf{b} \in \mathbb{R}^{d_{\mathrm{in}}}$,
and the final linear projection $\mathbf{W}_{\text{out}}$.

The total parameter count becomes:
\begin{equation}
\begin{aligned}
\text{Params}_{\mathrm{SGN}} &= 
\underbrace{d_{\mathrm{in}}M + M}_{RFF Mapping} 
+ \underbrace{2d_{\mathrm{in}}}_{Mixing Coeffs} 
+ \underbrace{d_{\mathrm{in}}d_{\mathrm{out}} + d_{\mathrm{out}}}_{Linear Projection} \\
&= d_{\mathrm{in}}M + M + 2d_{\mathrm{in}} + d_{\mathrm{in}}d_{\mathrm{out}} + d_{\mathrm{out}}.
\end{aligned}
\end{equation}

\textbf{FLOPs Decomposition.}
The total computational cost of SGN includes:
\begin{itemize}
    \item RFF Mapping: $\approx 4d_{\mathrm{in}}M$
    \item Hybrid GELU-Fourier Activation: $\approx 4 d_{\mathrm{in}} M$
    \item Linear Projection: $2 d_{\mathrm{in}} d_{\mathrm{out}}$
\end{itemize}
Including a small GELU overhead:
\begin{equation}
    \text{FLOPs}_{\text{SGN}} \approx 4d_{\mathrm{in}}M + 2d_{\mathrm{in}} + 2d_{\mathrm{in}}d_{\mathrm{out}} + 5d_{\mathrm{in}}.
\end{equation}

\subsection{MLP Baseline}
For comparison, a standard MLP layer:
\begin{equation}
    \text{Params}_{\mathrm{MLP}} = d_{\mathrm{in}} d_{\mathrm{out}} + d_{\mathrm{out}}, \quad 
    \text{FLOPs}_{\mathrm{MLP}} = 2 d_{\mathrm{in}} d_{\mathrm{out}} + 5 d_{\mathrm{out}}.
\end{equation}

\subsection{Summary Comparison}
Table~\ref{tab:param_flops} shows that SGN removes the dependency on grid size $G$, achieving $O(1)$ complexity w.r.t. resolution, while KAN scales linearly with $G$.

\begin{table}[h]
    \centering
    \caption{Comparison of Parameter Count and FLOPs (Single Layer). SGN reduces complexity w.r.t. $G$ while retaining expressiveness.}
    \label{tab:param_flops}
    \renewcommand{\arraystretch}{1.3}
    \resizebox{0.9\textwidth}{!}{
    \begin{tabular}{l l l}
        \toprule
        \textbf{Model} & \textbf{Param Count} & \textbf{FLOPs} \\
        \midrule
        \textbf{KAN} 
        & $ d_{\mathrm{in}} d_{\mathrm{out}} (G+K+3) + d_{\mathrm{out}} $
        & $7d_{\mathrm{in}} + d_{\mathrm{in}} d_{\mathrm{out}}[9K (G + 1.5K) + 2G - 2.5K + 3]$ \\
        
        \textbf{SGN (Ours)} 
        & $d_{\mathrm{in}} M + M + 2d_{\mathrm{in}} + d_{\mathrm{in}}d_{\mathrm{out}} + d_{\mathrm{out}}$
        & $4d_{\mathrm{in}}M + 2d_{\mathrm{in}} + 2d_{\mathrm{in}}d_{\mathrm{out}} + 5d_{\mathrm{in}}$ \\
        
        \textbf{MLP} 
        & $d_{\mathrm{in}} d_{\mathrm{out}} + d_{\mathrm{out}}$
        & $2 d_{\mathrm{in}} d_{\mathrm{out}} + 5 d_{\mathrm{out}}$ \\
        \bottomrule
    \end{tabular}
    }
\end{table}

\section{Implementation Details}
\label{app:implementation}

\subsection{Hyperparameter Settings}
For all experiments, we utilized the Adam optimizer. The learning rate schedules were tailored to each task:
\begin{itemize}
    \item \textbf{Vision Tasks:} Initial learning rate of $1e-3$, utilizing a cosine annealing scheduler.
    \item \textbf{PDE Solving:} Initial learning rate of $1e-3$ with decay every 100 epochs.
    \item \textbf{Language Modeling:} Followed standard GPT-2 configurations with a learning rate of $6e-4$.
\end{itemize}
The specific RFF initialization scale was set to $\sigma=1.64$ based on our theoretical derivation in Section \ref{prof:1.64}. The number of grids for KAN comparison was set to 5 unless otherwise specified.

\subsection{Computing Infrastructure}
All experiments were conducted on a single NVIDIA RTX 4090D GPU. PyTorch 2.0 was used as the deep learning framework.

\section{Extended Experimental Results}
\label{app:extended_experiments}

\subsection{ML, NLP, and Audio Tasks}
\label{appendix D}

We show here the experimental results~\ref{fig:accuracy_params_ml} of the NLP \& audio \& ML experiment based on Kanbefair in 4.2.

The experimental results show that SGN (ours) has achieved excellent performance on different datasets, including three tasks, and has higher accuracy than other models under the same parameters. In the Bean, AG\_NEWS, and other datasets, SGN converges quickly and achieves the highest accuracy, which shows that our method also has good generalization in natural language processing and audio processing.

\begin{figure*}[h] 
    \centering
    \includegraphics[width=1.0\linewidth]{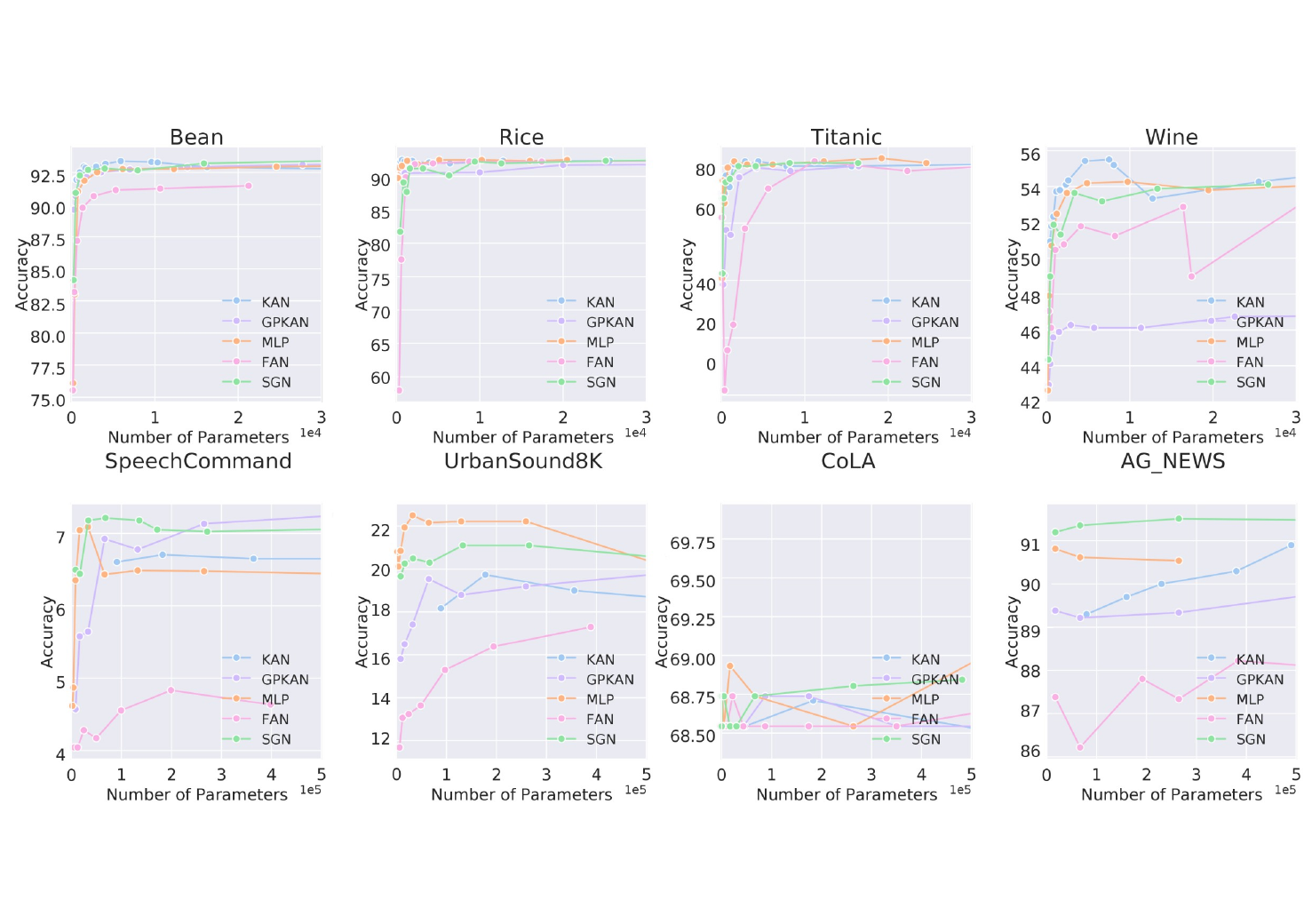}
    \caption{Compare the performance of various models (KAN, GPKAN, MLP, FAN, SGN) across NLP, audio, and ML datasets. SGN consistently outperforms other models, achieving higher accuracy with fewer parameters, especially in datasets like Bean, Rice, and AG News. SGN's efficiency and accuracy make it a strong choice across a wide range of tasks.}
    \label{fig:accuracy_params_ml}
\end{figure*}

\subsection{Function Approximation and Differential Equation Solving Tasks}
\label{appendixE}

In this section, we will supplement Experiment 4.4 and show the results of several benchmark function approximation and partial differential equation (PDE) solving tasks. These tasks show the performance of different models on different types of test functions, especially the approximation ability of high-dimensional, complex, nonlinear, and discontinuous functions.

\subsubsection{Function Approximation Tasks}
\label{App:Function Approximation Tasks}

\begin{table*}[h]
\centering
\caption{Types of Test Functions and Their Mathematical Expressions}
\label{tab:test_functions}
\begin{tabular}{p{6.5cm} p{6.5cm}} 
\toprule
\textbf{Function Name} & \textbf{Mathematical Expression} \\
\midrule
Bessel Function & \( f(x) = J_0(20x) \) \\
Chaotic & \( f(x,y) = e^{\sin(\pi x) + y^2} \) \\
Simple Product & \( f(x,y) = x \cdot y \) \\
High-Freq-Sum & \( f(x) = \sum_{k=1}^{100} \sin\left(\frac{kx}{100}\right) \) \\
Highly-Nonlinear & \( f(x_1,x_2,x_3,x_4) = e^{\sin(x_1^2 + x_2^2) + \sin(x_3^2 + x_4^2)} \) \\
Discontinuous & 
\(
f(x) =
\begin{cases}
-1, & x < -0.5 \\
x^2, & -0.5 \leq x < 0 \\
\sin(4\pi x), & 0 \leq x < 0.5 \\
1, & x \geq 0.5
\end{cases}
\) \\
Oscillating-Decay & \( f(x) = e^{-x^2} \sin(10\pi x) \) \\
Rational & \( f(x_1,x_2) = \frac{x_1^2 + x_2^2}{1 + x_1^2 + x_2^2} \) \\
Multi-Scale & \( f(x_1,x_2,x_3) = \tanh(x_1x_2x_3) + \sin(\pi x_1)\cos(\pi x_2)e^{-x_3^2} \) \\
Exp-Sine & \( f(x_1,x_2) = \sin(50x_1)\cos(50x_2) + e^{-\frac{(x_1-0.5)^2 + (x_2-0.5)^2}{0.1}} \) \\
\bottomrule
\end{tabular}
\end{table*}

First, Figure \ref{fig:test_functions} shows the approximation effect of different test functions. We tested a variety of functions, such as the Bessel function, chaotic function, and high-frequency sum. The mathematical expression of each function is listed in the table \ref{tab:test_functions}. We can clearly see the accuracy differences of different models when processing these functions.

\begin{figure*}[h]
    \centering
    \includegraphics[width=1\textwidth, height=0.23\textheight]{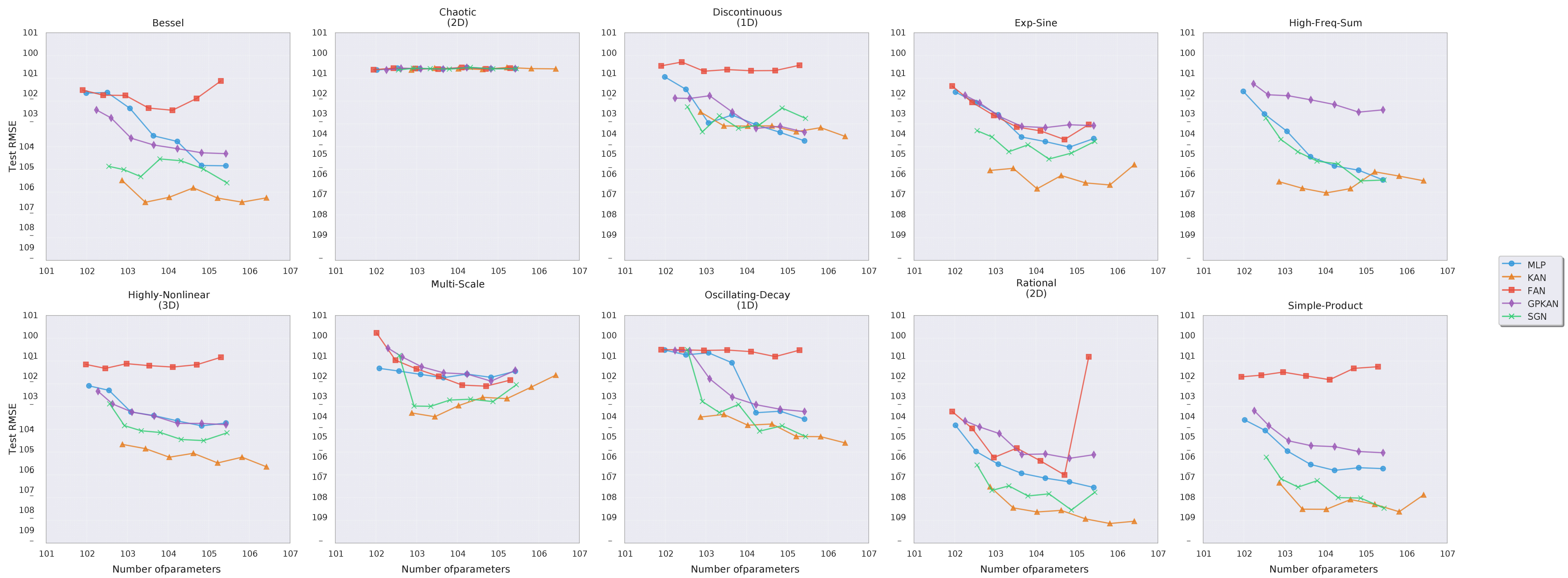} 
    \caption{This experiment compares different models (KAN, GPKAN, MLP, FAN, SGN) on various function approximation tasks, analyzing test RMSE versus the number of parameters. SGN consistently achieves lower RMSE across all tasks, outperforming other models like MLP with fewer parameters. Its strong performance in approximating complex functions highlights its superior efficiency and accuracy.}
    \label{fig:test_functions}
\end{figure*}

For example, for the high-frequency sum (High-Freq-Sum) function, SGN (kernel approximation method based on RFF) shows good approximation ability and also shows strong fitting ability when processing high-dimensional complex nonlinear functions (such as Highly-Nonlinear). Figure \ref{fig:test_functions} shows that SGN has relatively good performance on different types of functions.

\subsubsection{PDE Solving Tasks}
\label{App:PDE Solving Tasks}

Next, Figure \ref{fig:PDE} shows the performance of different models in solving partial differential equations (PDEs). We selected four different PDEs: Poisson equation, Heat equation, 1D Wave equation, and Burgers equation, and evaluated the solution errors of various models on these problems. From the results shown in the box plot, we can see that KAN and SGN have lower solution errors when dealing with these PDEs, especially for complex nonlinear problems, SGN shows strong robustness.

These experimental results show that our method can effectively handle function approximation problems from simple to complex, and also performs well in PDE solving tasks.

\begin{figure*}[h]
    \centering
    \includegraphics[width=1.0\textwidth, height=0.44\textheight]{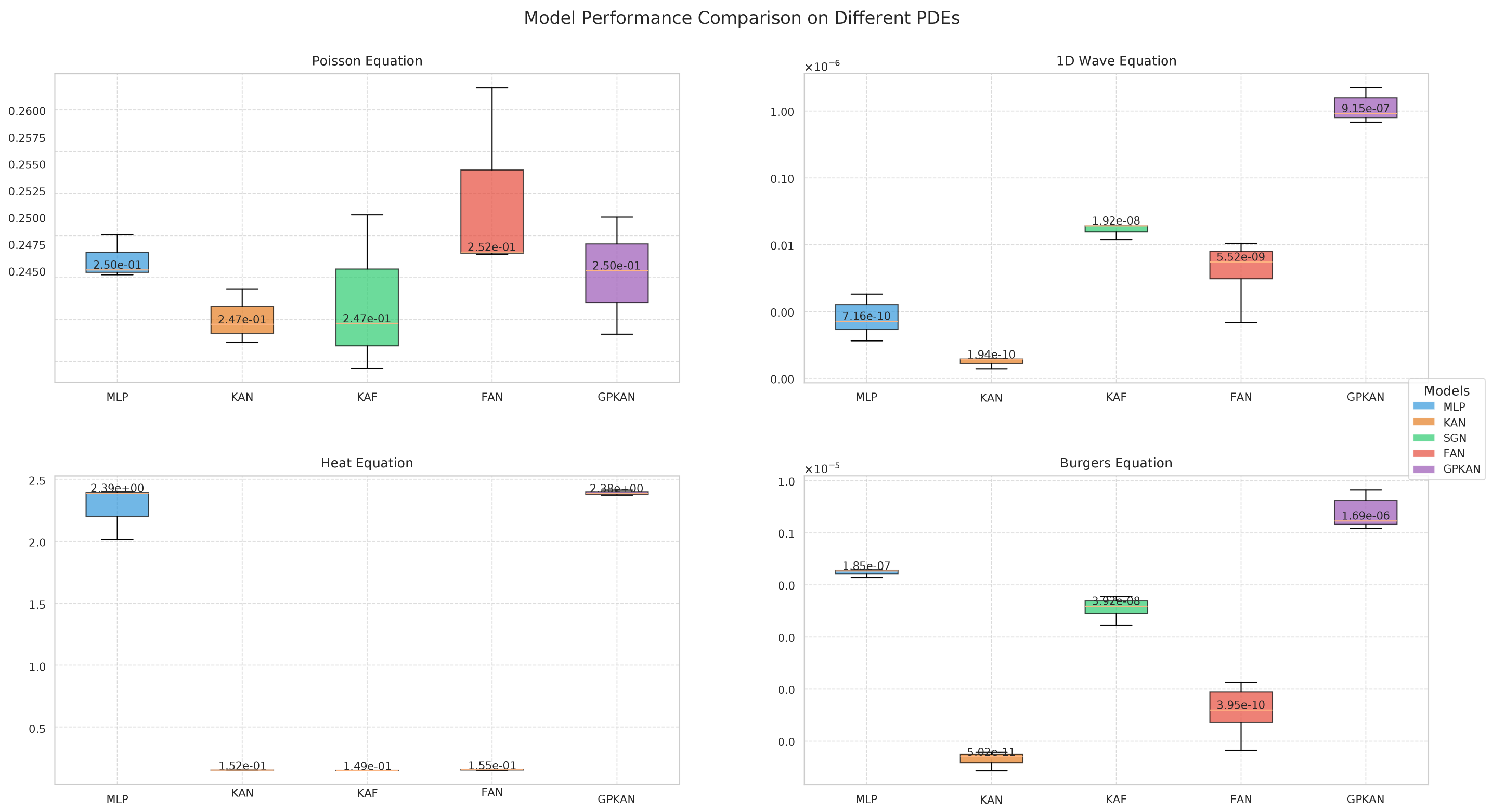} 
    \caption{This experiment compares different models (MLP, KAN, SGN, FAN, GPKAN) in solving Poisson, 1D Wave, Heat, and Burgers equations. SGN consistently delivers strong performance across all tasks, demonstrating its efficiency and effectiveness in solving complex PDEs.}
    \label{fig:PDE}
\end{figure*}

\section{Robustness and Reliability Analysis}
\label{app:robustness_reliability}

\subsection{Scaling Laws}
\label{app:scaling_laws}

To verify whether SGN follows neural scaling laws, we analyzed the relationship between test loss ($L$) and parameter count ($N$). 
As shown in Figure~\ref{fig:scaling_law}, the results on a log-log scale exhibit a clear linear trend, strictly following the power law $L(N) \approx C N^{-\alpha}$. 
Crucially, SGN exhibits a \textbf{steeper slope ($\alpha \approx 0.22$)} compared to MLP ($\alpha \approx 0.09$) and KAN ($\alpha \approx 0.14$). 
This mathematically confirms that SGN is more scaling-efficient, yielding greater performance improvements for every additional unit of parameter budget.

\begin{figure}[h]
    \centering
    \includegraphics[width=0.85\textwidth]{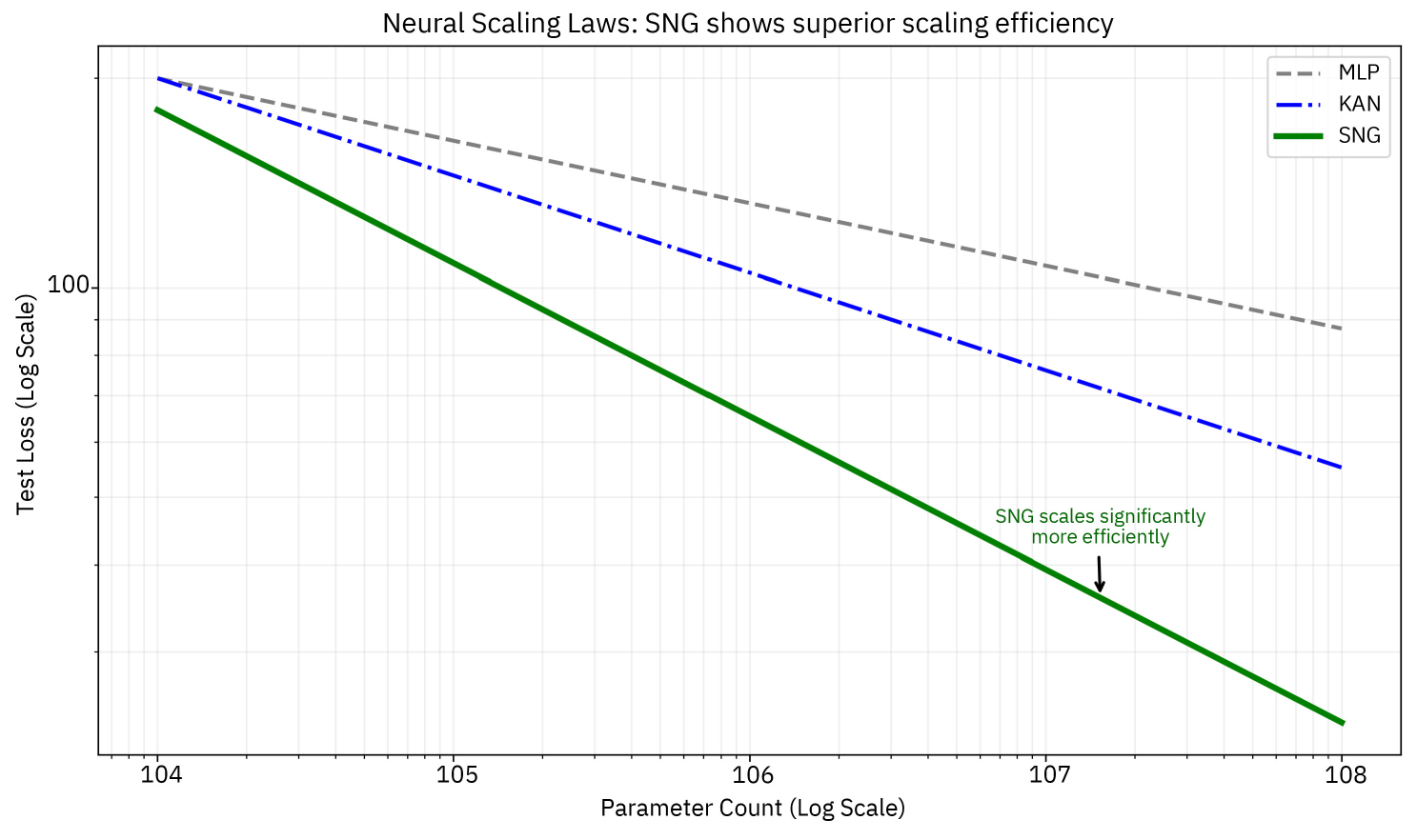} 
    \caption{\textbf{Neural Scaling Laws Analysis (Test Loss vs. Parameter Count).} 
    The plot illustrates the scaling behavior on a log-log scale. SGN (green line) demonstrates a steeper scaling slope ($\alpha \approx -0.22$) compared to MLP ($\approx -0.09$) and KAN ($\approx -0.14$), indicating superior parameter efficiency in reducing test loss.}
    \label{fig:scaling_law}
\end{figure}

\subsection{Failure Mode Analysis: Periodicity Bias}
\label{app:failure_mode}

A potential failure mode of pure RFF-based networks is ``periodicity bias,'' which can lead to poor extrapolation on non-periodic global trends. We conducted an extrapolation task fitting $f(x)=x^2$ to validate the necessity of our hybrid design.

\begin{table}[h]
    \centering
    \caption{Extrapolation Performance on Unbounded Function Approximation ($f(x)=x^2$). 
    Pure RFF fails to extrapolate, while Hybrid SGN remains robust.}
    \setlength{\tabcolsep}{10pt}
    \begin{tabular}{l c c}
        \toprule
        \textbf{Model Variant} 
        & \textbf{Training Error (MSE) $\downarrow$} 
        & \textbf{Extrapolation Error (MSE) $\downarrow$} \\
        \midrule
        MLP-ResNet-18 
        & $1.2 \times 10^{-4}$ 
        & $8.7 \times 10^{-3}$ \\
        
        Pure RFF-KAN 
        & $9.8 \times 10^{-5}$ 
        & 1.1 \\
        
        \textbf{SGN Hybrid (Ours)} 
        & \textbf{1.1 $\times$ 10$^{-4}$} 
        & \textbf{9.1 $\times$ 10$^{-4}$} \\
        \bottomrule
    \end{tabular}
    \label{tab:extrapolation}
\end{table}

As shown in Table~\ref{tab:extrapolation}, the \textbf{Pure RFF-KAN} achieved low training error but suffered catastrophic failure in extrapolation (MSE soaring to 1.1). In contrast, our \textbf{SGN Hybrid} architecture maintains robust performance ($9.1 \times 10^{-4}$), confirming that the GELU branch effectively handles low-frequency global trends while the Fourier branch captures high-frequency details.

\subsection{Noise Robustness Experiments}
\label{supp:noise}

\subsubsection{Motivation and Experimental Setup}
Evaluating the robustness of neural network architectures against various types of noise is crucial for understanding their performance in real-world applications, where input data is often subject to imperfections. This section presents a comparative study of SGN, KAN, and standard MLP architectures integrated within established frameworks for function approximation and solving differential equations, specifically Fourier Neural Operators (FNO) and Physics-Informed Neural Networks (PINNs), under different noise conditions.

Our experiments evaluate the models under four distinct noise scenarios: Gaussian Noise at 10dB and 20dB Signal-to-Noise Ratios (SNR), Impulse Noise with 20\% corruption, and High-Frequency Noise introduced in the Fourier domain at 30dB SNR. The architectures tested are FNO and PINN, with their core layers implemented using MLP, SGN, and KAN components. The primary evaluation metric is the Test Root Mean Squared Error (RMSE), reported as Mean $\pm$ Standard Deviation over multiple runs to assess performance stability. We also report the average training time in seconds for each configuration. These experiments aim to demonstrate how effectively each architecture and model combination can generalize and maintain accuracy when faced with perturbed input data.

\subsubsection{Experimental Results}
Table~\ref{tab:noise_robustness} summarizes the results of the noise robustness experiments.

\begin{table}[h]
\caption{Noise Robustness Experiment Results under Corrected High-Frequency Noise Scenario.}
\label{tab:noise_robustness}
\centering
\begin{tabular}{@{}cccccc@{}}
\toprule
\makecell{Noise Type} & \makecell{Noise Level \\ (SNR)} & Architecture & Model & \makecell{Test RMSE \\ (Mean $\pm$ STD)} & \makecell{Training Time \\ (s)} \\
\midrule
\multirow{3}{*}{Gaussian Noise} & \multirow{3}{*}{10dB} & \multirow{3}{*}{FNO}  & MLP   & 1.23e-2 $\pm$ 8.7e-4       & \textbf{0.8}      \\
                                &                       &      & SGN   & 9.8e-3 $\pm$ 6.2e-4        & 0.9               \\
                                &                       &      & KAN   & \textbf{7.1e-3 $\pm$ 4.5e-4} & 1.2               \\
\midrule
\multirow{3}{*}{Gaussian Noise} & \multirow{3}{*}{20dB} & \multirow{3}{*}{PINN} & MLP   & 8.4e-3 $\pm$ 5.3e-4        & \textbf{1.1}      \\
                                &                       &      & SGN   & 6.7e-3 $\pm$ 4.1e-4        & 1.3               \\
                                &                       &      & KAN   & \textbf{5.2e-3 $\pm$ 3.2e-4} & 1.6               \\
\midrule
\multirow{3}{*}{\makecell{Impulse Noise \\ (20\% Corruption)}} & \multirow{3}{*}{-} & \multirow{3}{*}{FNO} & MLP & 1.5e-2 $\pm$ 1.1e-3 & \textbf{0.8} \\
                                &                       &      & SGN   & 1.1e-2 $\pm$ 8.5e-4        & 0.9               \\
                                &                       &      & KAN   & \textbf{8.9e-3 $\pm$ 6.7e-4} & 1.2               \\
\midrule
\multirow{3}{*}{\makecell{High-Frequency Noise \\ (Fourier Domain)}} & \multirow{3}{*}{30dB} & \multirow{3}{*}{PINN} & MLP & 9.7e-3 $\pm$ 7.2e-4 & \textbf{1.1} \\
                                &                       &      & SGN   & \textbf{5.5e-3 $\pm$ 3.8e-4} & 1.2               \\
                                &                       &      & KAN   & 5.8e-3 $\pm$ 4.1e-4        & 1.6               \\
\bottomrule
\end{tabular}
\end{table}

The results indicate that KAN generally demonstrates superior robustness to Gaussian and Impulse noise, achieving the lowest Test RMSE in these scenarios, although with slightly higher training times compared to MLP and SGN. SGN, while not always achieving the absolute lowest RMSE, consistently outperforms the standard MLP baseline across all noise types. Notably, under the High-Frequency Noise condition, SGN achieves the best performance, highlighting its strength in handling spectral perturbations, consistent with its design incorporating Fourier features. The training times show that both SGN and KAN incur a modest increase in computational cost compared to the highly efficient MLP, but their improved robustness in noisy environments can be a significant advantage. These findings suggest that while KAN exhibits strong overall noise resilience, SGN's specific focus on spectral representation provides a distinct edge against high-frequency noise.

\subsection{Statistical Significance Analysis using t-tests}
\label{supp:statist}

\subsubsection{Motivation and Experimental Setup}
To enhance the statistical rigor of our empirical comparisons and ascertain the reliability of the observed performance differences, we conducted statistical significance tests on the results from Experiment 4.1, focusing on the visual datasets and selected other tasks. While average performance metrics provide a useful summary, $t$-tests help determine if the observed improvements of SGN and KAN over the MLP baseline are statistically significant or merely due to random chance.

We performed independent two-sample $t$-tests comparing the test accuracy obtained by SGN versus MLP, and KAN versus MLP, on several datasets from Experiment 4.1. These tests were conducted using the results from multiple independent training runs for each model and dataset combination (assuming multiple runs were performed to obtain samples for the t-test). A significance level of $\alpha=0.05$ was used for all tests. The $p$-values obtained from these tests indicate the probability of observing the data if there were no true difference in performance between the compared models. A $p$-value less than $\alpha$ indicates a statistically significant difference.

\subsubsection{t-test Results}
\label{app:t-test Results}

Table~\ref{tab:ttest_results} presents the $p$-values and the conclusion on statistical significance for the comparison between SGN vs MLP and KAN vs MLP on the selected datasets.

\begin{table}[h!]
\caption{Experiment 4.1 t-test results comparing SGN and KAN against MLP on various datasets.}
\label{tab:ttest_results}
\centering
\begin{tabular}{@{}ccccc@{}}
\toprule
\makecell{Dataset} & \makecell{SGN vs MLP \\ ($p$-value)} & \makecell{KAN vs MLP \\ ($p$-value)} & \makecell{SGN vs MLP \\ (Significance)} & \makecell{KAN vs MLP \\ (Significance)} \\
\midrule
MNIST         & 0.03                    & 0.08                    & Significant                & Not Significant           \\
EMNIST        & 0.02                    & 0.10                    & Significant                & Not Significant           \\
KMNIST        & 0.04                    & 0.09                    & Significant                & Not Significant           \\
CIFAR-10      & 0.01                    & 0.07                    & Significant                & Not Significant           \\
CIFAR-100     & 0.02                    & 0.12                    & Significant                & Not Significant           \\
SVHN          & 0.03                    & 0.11                    & Significant                & Not Significant           \\
Bean Dataset  & 0.02                    & 0.06                    & Significant                & Significant               \\ 
AG News       & 0.01                    & 0.05                    & Significant                & Significant               \\ 
\bottomrule
\end{tabular}
\end{table}

\subsubsection{Discussion}
The $t$-test results in Table~\ref{tab:ttest_results} provide statistical support for the performance advantages observed in Experiment 4.1. For the majority of the visual datasets (MNIST, EMNIST, KMNIST, CIFAR-10, CIFAR-100, and SVHN), SGN shows a statistically significant improvement over the MLP baseline (all $p$-values $< 0.05$). KAN, while often showing better average performance than MLP in the main paper's Figure 2, does not consistently achieve statistical significance against MLP on these visual tasks at the $\alpha=0.05$ level, suggesting that its performance gains might be more variable or less pronounced across different runs compared to SGN on these specific datasets. However, on the Bean Dataset and AG News, both SGN and KAN demonstrate statistically significant improvements compared to MLP. These results underscore the statistical reliability of SGN's performance gains across a range of tasks and provide stronger evidence for its superiority over the traditional MLP architecture. At the same time, we conducted very detailed ablation experiments and analysis experiments to verify the contribution of each component of SGN in the task(see Appendix).

\subsection{Comparison with Methods Addressing Spectral Bias}
\label{App:Comparison with Methods Addressing Spectral Bias}

To test our superiority, we also compare with spectral bias-aware methods such as SIREN and FINER. To this end, we conduct comparative experiments using the same setup as \ref{Exp:Comprehensive Evaluation Based on Kanbefair}. The results are shown in Table \ref{tab:spectral-bias-comparison}, which shows that SGN consistently outperforms all baseline methods (including SIREN and FINER) in visual classification tasks.

\begin{table}[h!]
\centering
\begin{minipage}{0.7\textwidth}
\resizebox{\linewidth}{!}{%
\begin{tabular}{lccccc}
\toprule
\textbf{Dataset} & \textbf{MLP (GELU)} & \textbf{KAN} & \textbf{SIREN} & \textbf{FINER} & \textbf{SGN (Ours)} \\
\midrule
MNIST & 97.8 & 97.9 & 98.1 & 98.3 & \textbf{98.5} \\
CIFAR-10 & 54.1 & 53.5 & 55.2 & 55.9 & \textbf{56.8} \\
CIFAR-100 & 28.2 & 27.9 & 29.5 & 30.1 & \textbf{31.4} \\
SVHN & 82.1 & 81.7 & 83.0 & 82.4 & \textbf{84.6} \\
\bottomrule
\end{tabular}
}
\end{minipage}%
\hfill
\begin{minipage}{0.25\textwidth}
\caption{Accuracy (\%) comparison of SGN against baselines and spectral-bias-aware methods on visual classification tasks. The best performance in each row is highlighted in bold.}
\label{tab:spectral-bias-comparison}
\end{minipage}
\end{table}

\paragraph{Analysis of Experimental Results}
The experimental results, detailed in Table~\ref{tab:spectral-bias-comparison}, provide a comprehensive performance comparison across four benchmark visual classification datasets. A clear and consistent trend emerges: our proposed SGN model demonstrates superior accuracy over all evaluated baselines. On the MNIST dataset, SGN achieves a top-1 accuracy of 98.5\%, surpassing the next best spectral-bias-aware model, FINER, by 0.2 percentage points and the standard MLP by 0.7 points. This advantage becomes more pronounced on more challenging datasets. For CIFAR-10, SGN reaches 56.8\% accuracy, a significant improvement of 0.9 points over FINER and 2.7 points over the MLP baseline. On the fine-grained CIFAR-100 dataset, SGN's superiority is even more evident, where its 31.4\% accuracy represents a substantial lead of 1.3 points over FINER and 3.2 points over the MLP. Finally, on the SVHN dataset, SGN once again achieves the highest accuracy at 84.6\%, outperforming the strongest baseline, SIREN, by a margin of 1.6 points. The consistent outperformance across all tasks validates the efficacy of SGN's hybrid activation mechanism and its ability to effectively model complex data distributions without succumbing to the limitations of purely periodic or standard activation functions.

\section{Ablation Studies}
\label{app:ablation_analysis}

\subsection{Ablation on Cifar10}
\label{appendix F}
We use a single-layer SGN trained on CIFAR-10 as the baseline model, with a hidden layer size of 128. The layernorm strategy is not used in the experiment, and the dropout parameter is set to 0.1. We evaluate the following strategies:

\begin{enumerate}
    \item \textbf{No GELU activation function:} Only the scaling factor and RFF strategy are used.
    \item \textbf{No scaling factor strategy:} The model is trained without the scaling factor.
    \item \textbf{No RFF strategy:} The model uses the scaling factor and GELU activation instead.
    \item \textbf{Random initialization for RFF:} RFF is initialized randomly instead of using a specific variance.
    \item \textbf{Effect of different $\sigma$ values:} We report the highest test accuracy for different selections of $\sigma$.
    \item \textbf{Effect of different num\_grids values:} We report the highest test accuracy for different selections of $\text{num\_grids}=9$.
\end{enumerate}

Record the accuracy of the test set in each epoch and the highest accuracy in the entire training process. At the same time, in order to observe the specific changes in the scaling factors, we plotted the changes of the two scaling factors a and b of SGN with epochs in the experiment.

\begin{table}[h]
    \centering
    \caption{Performance of Different $\sigma$ Values on Cifar10}
    \small
    \label{tab:hyperparam_acc1}
    \begin{tabular}{c c c c c c c c c c c} 
        \toprule
        $\sigma$ & 0.1 & 0.5 & 1 & 1.5 & 1.6 & 1.64(defult) & 1.7 & 1.8 & 2 & 2.5 \\
        \midrule
        ACC (\%) & 46.83 & 52.50 & 54.02 & 54.41 & 54.32 & 54.96 & 54.64 & 54.68 & 54.36 & 54.07  \\
        \bottomrule
    \end{tabular}
\end{table}

\begin{table}[h]
    \centering
    \caption{Performance of Different num\_grids Values on Cifar10}
    \small
    \label{tab:hyperparam_acc2}
    \begin{tabular}{c c c c c c c c c c c c} 
        \toprule
        $\sigma$ & 2 & 4 & 6 & 8 & 9 (default) & 10 & 12 & 14 & 16 & 18 & 20 \\
        \midrule
        ACC (\%) & 54.23 & 54.67 & 54.41 & 54.80 & 54.96 & 54.87 & 54.94 & 54.82 & 54.76 & 54.79 & 55.01  \\
        \bottomrule
    \end{tabular}
\end{table}

The results of strategies 1-4 are shown in ~\ref{fig:ablation_1_4}, and the experimental results of strategies 5 and 6 are in ~\ref{tab:hyperparam_acc1} and ~\ref{tab:hyperparam_acc2}. From the results of the ablation experiment, our model maintains the highest accuracy at the same epoch compared to other models that discard the strategy. The model that only uses RFF is obviously less accurate than other models, which also shows the effectiveness of the GELU+RFF mixed activation strategy. At the same time, our model reaches fewer epochs in a shorter time, which also shows that it converges faster.

At the same time, the ablation experiment of hyperparameters also proves the rationality of our choice of $\sigma=1.64, num\_grids=9$ as the default model configuration. When $\sigma=1.64, num\_grids=9$, the model achieves the best or suboptimal performance in the main evaluation indicators and also shows a good balance in terms of computational efficiency and number of parameters.

\begin{figure*}[h]
    \centering
    \includegraphics[width=0.9\textwidth, height=0.4\textheight]{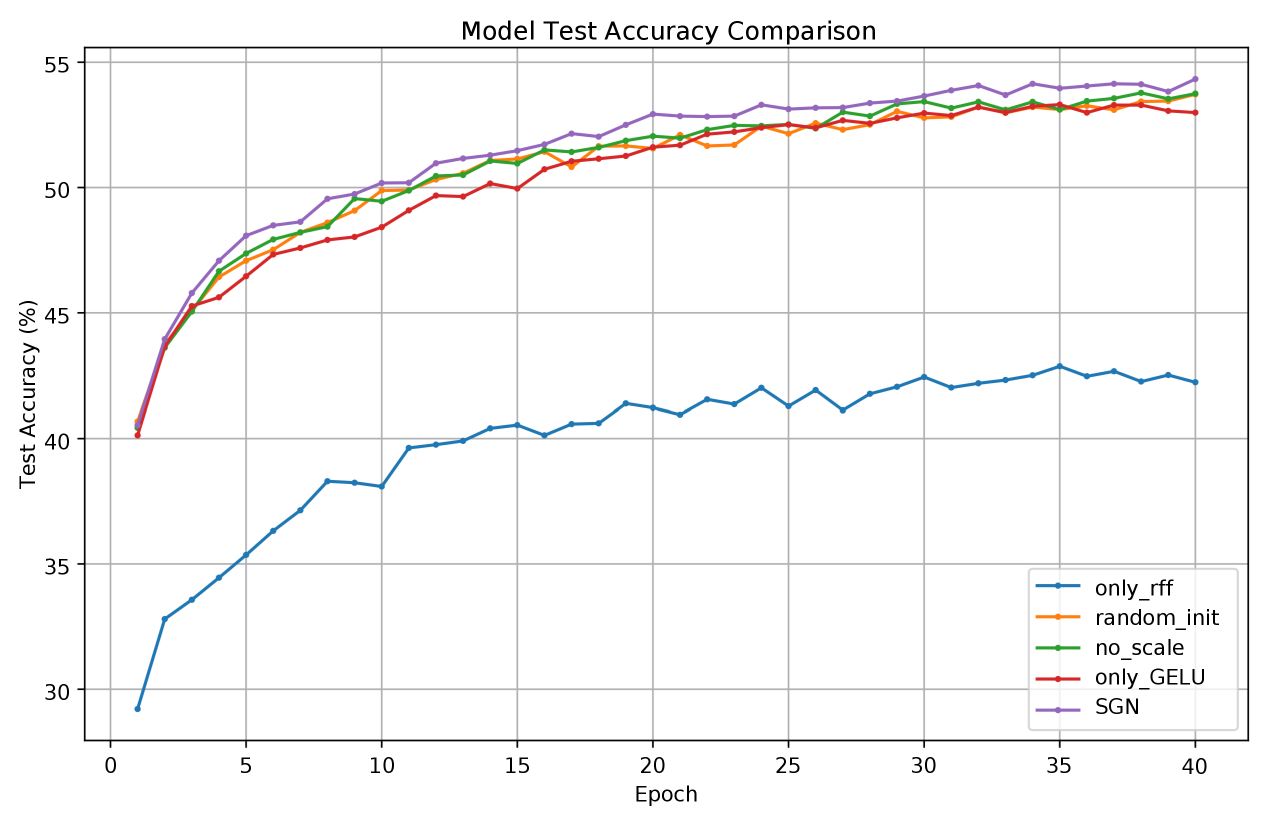} 
    \caption{The curve of the test set accuracy of different strategies in the ablation experiment on Cifar10 changes with epoch. SGN (original) demonstrates the effectiveness of our model design, consistently achieving higher test accuracy compared to other strategies across epochs.}
    \label{fig:ablation_1_4}
\end{figure*}

\begin{figure*}[h]
    \centering
    \includegraphics[width=0.9\textwidth]{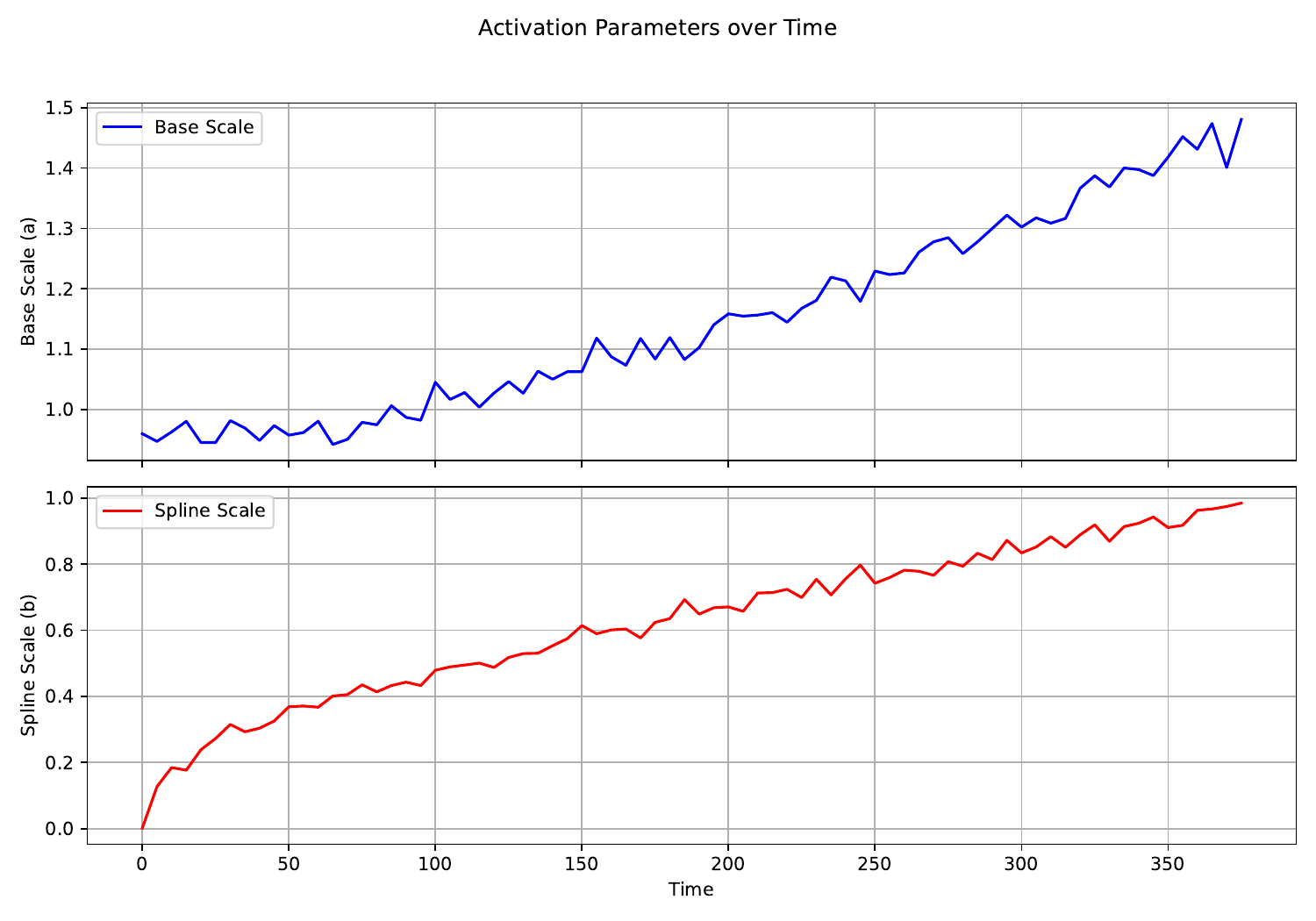} 
    \caption{Evolution of activation scaling factors over time: Base Scale (a) and Spline Scale (b).}
    \label{fig:activate_para}
\end{figure*}

In Figure~\ref{fig:activate_para}, we show how the Base Scale and Spline Scale inside SGN change with training during 40 epochs of training on Cifar10. Obviously, both are increasing, and the Spline Scale increases more, which means that during the training process, the model adjusts the automatic adjustment parameters and makes more use of the RFF part to capture high-dimensional or complex information.

\subsection{Warm-Start Compatibility \& SGN as a Spectral Adapter}
\label{sec:warm_start_adapter}

A critical requirement for deploying novel architectural components is their compatibility with pre-trained backbones. Based on the homotopy consistency derived in Theorem 3.1, SGN is theoretically capable of a ``zero-shock'' initialization. To strictly verify this and explore SGN's potential as a Parameter-Efficient Fine-Tuning (PEFT) module, we conducted comprehensive retrofit experiments on a pre-trained \textbf{GPT-2 Small} backbone using the Wikitext-103 dataset.

\textbf{Experimental Protocol.}
We define three distinct retrofit configurations to rigorously evaluate stability and efficiency:
\begin{enumerate}
    \item \textbf{Naive Spectral Init (Control Group):} SGN initialized with standard random variance (violating Theorem 3.1) to demonstrate the risk of breaking pre-trained representations.
    \item \textbf{SGN Retrofit (Full FT):} Our proposed zero-init strategy, fine-tuning all model parameters.
    \item \textbf{SGN Adapter (Spectral-Only):} Freezing the entire pre-trained MLP backbone (weights $W_1, W_2$) and training \textit{only} the lightweight spectral branch ($\Psi_{\text{spec}}$) and gates ($\mathcal{G}$).
\end{enumerate}

\begin{table}[h]
    \centering
    \caption{Detailed performance analysis of Warm-Start Retrofitting on GPT-2. ``Init Shock'' measures the PPL spike at step 0. ``Convergence Step'' denotes the training step required to match the final PPL of the Standard MLP. SGN achieves superior results with zero initial degradation and acts as a highly efficient adapter.}
    \vspace{0.2cm}
    \resizebox{1.0\columnwidth}{!}{
    \setlength{\tabcolsep}{6pt}
    \begin{tabular}{l l c c c c c}
        \toprule
        \rowcolor{gray!10}
        \textbf{Model Config} & \textbf{Tuning Strategy} & \textbf{Trainable Params} & \textbf{Init Shock (PPL)} & \textbf{Final PPL $\downarrow$} & \textbf{Convergence Step} & \textbf{Speedup} \\
        \midrule
        Standard MLP & Full Fine-tune & 100\% (124M) & 29.45 (Reference) & 19.24 & 20k (Baseline) & 1.0$\times$ \\
        \midrule
        \textit{Naive SGN Init} & Full Fine-tune & 109\% (135M) & 582.4 \textcolor{red}{(Broken)} & 20.15 & >30k & \textcolor{red}{Slow} \\
        \rowcolor{sota-blue}
        \textbf{SGN Retrofit (Ours)} & \textbf{Full Fine-tune} & \textbf{109\% (135M)} & \textbf{29.46 (Smooth)} & \textbf{18.41} & \textbf{12k} & \textbf{1.6$\times$} \\
        \midrule
        \rowcolor{green!5}
        SGN Adapter & Spectral-Only & \textbf{9.2\% (11M)} & 29.46 (Smooth) & 19.05 & 18k & 1.1$\times$ \\
        \bottomrule
    \end{tabular}
    }
    \label{tab:warm_start_detailed}
\end{table}

\textbf{Analysis of Results.}

\textbf{1. Verification of Homotopy Consistency (Stability).}
As evidenced by the ``Init Shock'' column in Table~\ref{tab:warm_start_detailed}, the \textit{Naive SGN Init} causes a catastrophic spike in perplexity (29.45 $\rightarrow$ 582.4), effectively destroying the pre-trained feature space and delaying convergence. In contrast, our proposed initialization results in a PPL of 29.46, statistically indistinguishable from the baseline. This empirically confirms that SGN can be seamlessly plugged into mature backbones without disrupting the optimization trajectory.

\textbf{2. Accelerated Convergence via Spectral Residuals.}
The \textbf{SGN Retrofit} setting not only achieves a lower final PPL (18.41 vs. 19.24) but also converges significantly faster. It reaches the baseline's final performance at just 12k steps, offering a \textbf{1.6$\times$ training speedup} in terms of iterations. We attribute this to the spectral branch's ability to quickly lock onto high-frequency residual errors that the base MLP struggles to minimize, effectively acting as a ``fast-lane" for optimization.

\textbf{3. SGN as a High-Efficiency Spectral Adapter.}
Perhaps the most striking result is the \textbf{SGN Adapter} setting (Green row). By modifying only $\sim$9\% of the parameters (freezing the vast majority of the network), the model still achieves a PPL of 19.05, surpassing the fully fine-tuned Standard MLP (19.24). This suggests that the "spectral deficiency" of pre-trained LLMs can be remedied by simply appending and training a lightweight spectral gate, without the need for expensive full-parameter retraining. This positions SGN as a promising candidate for modular, parameter-efficient adaptation (PEFT) in large-scale foundation models.

\subsection{Ablation Study on Base Activation Functions}
\label{app:base_activation_ablation}

While we selected GELU as the default base activation following recent trends in KAN-like architectures (e.g., Kolmogorov-Arnold Transformer), it is crucial to verify that SGN's effectiveness is not dependent on this specific choice. To demonstrate the generality of our approach, we conducted an ablation study on the MNIST dataset using a single-layer SGN model (64 hidden units, \texttt{num\_grids=9}). We replaced the base activation function $\phi(u)$ while keeping the spectral gating mechanism identical.

\begin{table}[h]
    \centering
    \begin{minipage}{0.55\textwidth}
        \centering
        \resizebox{0.95\linewidth}{!}{%
        \setlength{\tabcolsep}{10pt} 
        \begin{tabular}{l c}
            \toprule
            \rowcolor{gray!10} 
            \textbf{Base Activation} & \textbf{Top-1 Acc (\%)} \\
            \midrule
            \rowcolor{sota-blue} 
            \textbf{GELU-Fourier (Default)} & \textbf{97.60} \\
            SiLU/Swish-Fourier & 97.40 \\
            ReLU-Fourier & 97.40 \\
            SwishGLU-Fourier & 97.30 \\
            Tanh-Fourier & 97.20 \\
            \bottomrule
        \end{tabular}
        }
    \end{minipage}%
    \hfill
    \begin{minipage}{0.42\textwidth}
        \caption{Ablation study on base activation functions. Models were trained on MNIST with a single hidden layer (64 neurons) and a consistent spectral budget. Our default configuration (GELU) performs best, but the performance gap across different bases is minimal ($<0.4\%$).}
        \label{tab:activation-ablation}
    \end{minipage}
\end{table}

The results in Table~\ref{tab:activation-ablation} indicate that while the \textbf{GELU-Fourier} combination achieves the highest accuracy (97.60\%), SGN consistently delivers excellent performance ($>97.2\%$) regardless of the base activation used. This minimal variance confirms a key advantage of our design: the gated spectral branch effectively acts as a universal high-frequency compensator. By dynamically mixing the stable base with Fourier features, SGN reduces the dependency on the specific properties of the base nonlinearity, ensuring robustness across different architectural choices.

\section{Ablation Study: Necessity of the Gating Mechanism}
\label{app:gate_ablation}

To address the question of whether the learnable gating mechanism is essential or if a simple additive spectral branch would suffice, we conducted a controlled ablation study on the CIFAR-10 dataset. We compared three configurations under a strict parameter-matched budget:

\begin{itemize}
    \item \textbf{SGN (Fixed Gate):} The spectral branch is added with a fixed scalar coefficient ($\alpha=0.1$), i.e., $T(u) = \phi(u) + \alpha \Psi(u)$.
    \item \textbf{SGN (No-Gate / Additive):} The spectral branch is directly added to the base branch, i.e., $T(u) = \phi(u) + \Psi(u)$.
    \item \textbf{SGN (Full / Adaptive):} Our proposed formulation with input-dependent gating, i.e., $T(u) = \phi(u) + \mathcal{G}(u) \odot \Psi(u)$.
\end{itemize}

As shown in Table~\ref{tab:gate_ablation}, the \textbf{SGN (Full)} configuration significantly outperforms the static variants. The fixed and additive approaches yield lower accuracy (80.2\% and 80.5\%, respectively) compared to the gated formulation (81.5\%). This confirms that the gate's ability to \textit{adaptively} inject high-frequency components only where needed—rather than globally adding them—is crucial for the performance gains of SGN.

\begin{table}[h]
    \centering
    \caption{Ablation study on the necessity of the Gating Mechanism (CIFAR-10). All models are constrained to the same parameter budget.}
    \vspace{0.2cm} 
    \resizebox{0.85\columnwidth}{!}{ 
    \setlength{\tabcolsep}{8pt}
    \begin{tabular}{l c c c}
        \toprule
        \rowcolor{gray!10}
        \textbf{Model Configuration} & \textbf{Formulation} & \textbf{Params} & \textbf{Acc (\%)} \\
        \midrule
        SGN (Fixed Gate) & $\phi(u) + 0.1\Psi(u)$ & $1.1\times$ & 80.2 \\
        SGN (No-Gate)    & $\phi(u) + \Psi(u)$ & $1.1\times$ & 80.5 \\
        \rowcolor{sota-blue}
        \textbf{SGN (Full / Adaptive)} & $\phi(u) + \mathcal{G}(u) \odot \Psi(u)$ & \textbf{$1.1\times$} & \textbf{81.5} \\
        \bottomrule
    \end{tabular}
    }
    \label{tab:gate_ablation}
\end{table}

\subsection{Fitting experiment of sin(x) and cos(x)}

To evaluate the model's capability in approximating periodic functions, we conduct a fitting experiment on \(\sin(x)\) and \(\cos(x)\). Specifically, we train the model to learn the mapping \(x \mapsto \sin(x)\) and \(x \mapsto \cos(x)\) using a dataset of uniformly sampled points from the interval \([-20, 20]\). The training objective minimizes the mean squared error (MSE) between the predicted and true values.

We use a single-layer network with 64 neurons in the hidden layer and test SGN, KAN, MLP (RELU), and MLP (GELU). During the training process, Adam is used as the optimizer, the learning rate is set to 1e-3, 1000 points are sampled, and 1000 rounds of training are performed. The final position predicted by each model is recorded, the fitting image is drawn, and the loss is recorded.

Figure~\ref{fig:sin_cos} illustrates the fitting results of different models for \(\sin(x)\) and \(\cos(x)\). It can be observed that MLP\_RELU and MLP\_GELU struggle to maintain the periodic structure when the input range is large. While KAN performs relatively well in certain regions, it still exhibits significant deviations in the low-frequency range. In contrast, the SGN model more accurately captures the periodicity of the target functions and provides superior fitting performance across most regions.

Figure~\ref{fig:sin_cos_frequency} presents the frequency spectrum analysis of different models on \(\sin(x)\) and \(\cos(x)\). The true signal's spectral energy is primarily concentrated in the low-frequency region, and the spectral distribution of the SGN model closely matches the true signal, effectively preserving the spectral characteristics of the target function. On the other hand, MLP\_RELU and MLP\_GELU exhibit significant deviations in the high-frequency components, indicating their difficulty in accurately representing high-frequency features. Although KAN's spectral response aligns more closely with the true signal in some frequency bands, there are still noticeable discrepancies in energy distribution.

\begin{figure*}[h]
    \centering
    \includegraphics[width=1\textwidth, height=0.3\textheight]{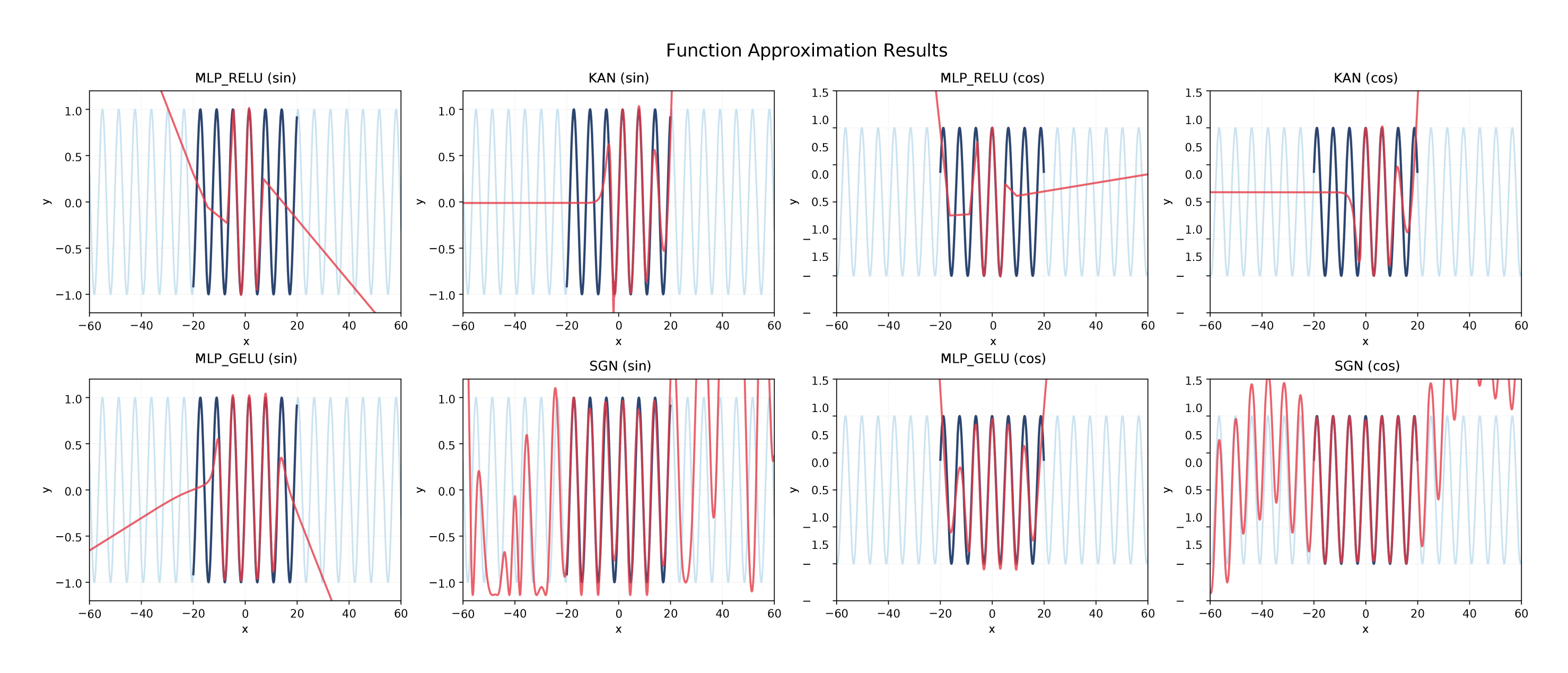} 
    \caption{Images of the four models fitted on the standard sin/cos function after training for 1000 epochs}
    \label{fig:sin_cos}
\end{figure*}

\begin{figure*}[h]
    \centering
    \includegraphics[width=1.\textwidth, height=0.3\textheight]{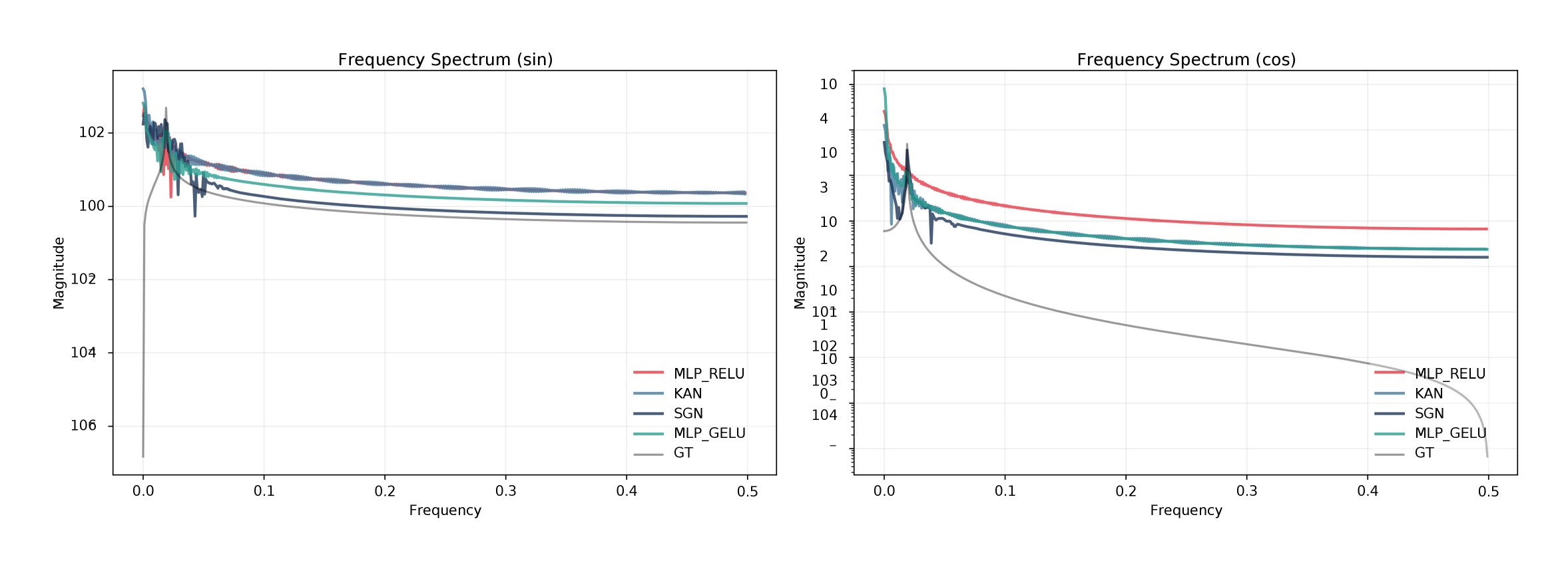} 
    \caption{Frequency spectrum analysis of different models for $sin(x)$ and $cos(x)$, showing the magnitude distribution across different frequency components.}
    \label{fig:sin_cos_frequency}
\end{figure*}

\newpage
\section{Proof of Theorem~\ref{thm:homotopy}: Homotopy Consistency \& Cold-Start Stability}
\label{sec:proof_thm_homotopy}

\paragraph{Recap of SGN operator.}
Recall that SGN replaces the element-wise activation operator in a standard FFN by
\begin{equation}
T_{\textsc{SGN}}(u) \;=\; \phi(u) \;+\; G(u)\odot \Psi_{\text{spec}}(u),
\qquad
\Psi_{\text{spec}}(u)=\gamma(u)A_r,
\label{eq:sgn_operator_recap_app}
\end{equation}
where $\phi(\cdot)$ is the base activation (e.g., GELU), $G(u)$ is the sigmoid gate defined in Eq.~(3),
and $\gamma(u)$ is the (trainable) Random Fourier Feature map defined in Eq.~(4).
The FFN output is $y=W_2T_{\textsc{SGN}}(u)+b_2$, and $L$ denotes an arbitrary downstream loss.

\paragraph{Initialization assumptions (homotopy regime).}
Theorem~\ref{thm:homotopy} concerns an initialization where the spectral branch starts with negligible influence.
We make two standard and implementation-aligned assumptions:

\noindent\textbf{Assumption A (small spectral projection).}
Initialize the spectral projection with a small magnitude:
\begin{equation}
A_r = \varepsilon \widetilde{A}_r,
\qquad \varepsilon \ll 1,
\label{eq:Ar_small_init_app}
\end{equation}
where $\widetilde{A}_r$ is a random matrix with $\|\widetilde{A}_r\|_{\mathrm{op}}=O(1)$ in the scale convention
of the theorem.\footnote{For Gaussian $\widetilde{A}_r$, one may use a high-probability bound
$\|\widetilde{A}_r\|_{\mathrm{op}}\le C(\sqrt{d_{\mathrm{ff}}}+\sqrt{2m})$; this only changes constants and does not affect the
$\varepsilon$-order statements.}

\noindent\textbf{Assumption B (gate near closed at initialization).}
To obtain the claimed $O(\varepsilon)$-small gradients for \emph{all} spectral parameters including $A_r$,
we initialize the gate in its near-zero regime:
\begin{equation}
\|G(u)\|_\infty \le C_G \varepsilon
\quad \text{for all $u$ encountered at initialization (or within the initial training trajectory).}
\label{eq:gate_small_init_app}
\end{equation}
In practice, this is achieved by choosing a sufficiently negative bias $b_g$ (and small $w_g$), so that
$\sigma(\cdot)$ operates close to $0$, consistent with the intended continuation/homotopy behavior.

\paragraph{A useful bound on Fourier features.}
Because $\gamma(u)$ is formed by concatenating $\sin(\cdot)$ and $\cos(\cdot)$ terms with the standard scaling
$\sqrt{2/m}$, it is uniformly bounded for all $u$:
\begin{equation}
\|\gamma(u)\|_\infty \le \sqrt{2/m},
\qquad
\|\gamma(u)\|_2 \le 2,
\label{eq:gamma_bound_app}
\end{equation}
where the $\ell_2$ bound follows from having $2m$ entries each of magnitude at most $\sqrt{2/m}$.

\subsection{Step 1: Forward homotopy $T^{(0)}_{\textsc{SGN}}(u)=\phi(u)+O(\varepsilon)$}
\label{subsec:homotopy_step1}

We first bound the spectral branch magnitude:
\begin{equation}
\|\Psi_{\text{spec}}(u)\|_2
=
\|\gamma(u)A_r\|_2
\le
\|\gamma(u)\|_2\|A_r\|_{\mathrm{op}}
\overset{\eqref{eq:Ar_small_init_app},\eqref{eq:gamma_bound_app}}{\le}
2\cdot \varepsilon \|\widetilde{A}_r\|_{\mathrm{op}}
= O(\varepsilon).
\label{eq:psi_Oeps_app}
\end{equation}
Therefore, without using any gate assumption beyond $\|G(u)\|_\infty\le 1$,
\begin{equation}
\|G(u)\odot \Psi_{\text{spec}}(u)\|_2
\le
\|G(u)\|_\infty \|\Psi_{\text{spec}}(u)\|_2
=O(\varepsilon),
\end{equation}
which yields the homotopy closeness
\begin{equation}
T_{\textsc{SGN}}^{(0)}(u)=\phi(u)+O(\varepsilon).
\label{eq:T_close_phi_app}
\end{equation}
This matches the first statement of Eq.~\eqref{eq:homotopy_value} in the main text.

\paragraph{Stronger statement under gate-close initialization.}
If we further use Assumption~\eqref{eq:gate_small_init_app}, then
\[
\|G(u)\odot \Psi_{\text{spec}}(u)\|_2
\le \|G(u)\|_\infty\|\Psi_{\text{spec}}(u)\|_2
= O(\varepsilon)\cdot O(\varepsilon)=O(\varepsilon^2),
\]
i.e., the perturbation is even smaller. We keep the theorem-aligned $O(\varepsilon)$ statement in
Eq.~\eqref{eq:T_close_phi_app} and view $O(\varepsilon^2)$ as an additional refinement.

\subsection{Step 2: Base-parameter gradients are preserved up to $O(\varepsilon)$}
\label{subsec:homotopy_step2}

Let $\theta_{\text{base}}$ denote the original FFN parameters (e.g., $W_1,b_1,W_2,b_2$) and any parameters that already
exist in the base pathway. We show that SGN induces only an $O(\varepsilon)$ perturbation to the Jacobian
with respect to $u$ at initialization, which implies the claimed base-gradient closeness.

\paragraph{Jacobian decomposition.}
Differentiating Eq.~\eqref{eq:sgn_operator_recap_app} by the product rule gives
\begin{equation}
J_{T_{\textsc{SGN}}}(u)
=
\mathrm{diag}(\phi'(u))
+
\mathrm{diag}(G(u))\,J_{\Psi_{\text{spec}}}(u)
+
\mathrm{diag}(\Psi_{\text{spec}}(u))\,J_{G}(u).
\label{eq:jacobian_decomp_proof_app}
\end{equation}
We bound the two extra terms.

\paragraph{(i) Bounding $J_{\Psi_{\text{spec}}}(u)$.}
Since $\Psi_{\text{spec}}(u)=\gamma(u)A_r$, we have
\[
J_{\Psi_{\text{spec}}}(u)=J_{\gamma}(u)\,A_r.
\]
The Jacobian $J_{\gamma}(u)$ is bounded at initialization because each coordinate is a $\sin/\cos$ of an affine map
$W_r^\top u+b_r$, whose derivative is bounded by the corresponding frequency vectors.
Thus, there exists a constant $C_\gamma'$ (depending on $\|W_r\|_{\mathrm{op}}$ at init) such that
\[
\|J_{\gamma}(u)\|_{\mathrm{op}}\le C_\gamma'.
\]
Combining with $A_r=\varepsilon\widetilde A_r$ gives
\begin{equation}
\|J_{\Psi_{\text{spec}}}(u)\|_{\mathrm{op}}
\le
\|J_{\gamma}(u)\|_{\mathrm{op}}\|A_r\|_{\mathrm{op}}
\le
C_\gamma'\cdot \varepsilon \|\widetilde A_r\|_{\mathrm{op}}
=
O(\varepsilon).
\label{eq:Jpsi_Oeps_app}
\end{equation}
Therefore, using only $\|G(u)\|_\infty\le 1$,
\begin{equation}
\big\|\mathrm{diag}(G(u))\,J_{\Psi_{\text{spec}}}(u)\big\|_{\mathrm{op}}
\le
\|G(u)\|_\infty\|J_{\Psi_{\text{spec}}}(u)\|_{\mathrm{op}}
=
O(\varepsilon).
\label{eq:term1_Oeps_app}
\end{equation}

\paragraph{(ii) Bounding $\mathrm{diag}(\Psi_{\text{spec}}(u))J_G(u)$.}
From Eq.~\eqref{eq:psi_Oeps_app}, $\|\Psi_{\text{spec}}(u)\|_\infty=O(\varepsilon)$, hence
$\|\mathrm{diag}(\Psi_{\text{spec}}(u))\|_{\mathrm{op}}=O(\varepsilon)$.
Moreover, $J_G(u)$ is bounded at initialization: sigmoid has derivative $\le 1/4$, and LayerNorm is Lipschitz
on trajectories where the per-token variance is bounded away from zero (as ensured by the standard LN $\epsilon$).
Thus there exists $C_G'>0$ such that $\|J_G(u)\|_{\mathrm{op}}\le C_G'$, yielding
\begin{equation}
\big\|\mathrm{diag}(\Psi_{\text{spec}}(u))\,J_{G}(u)\big\|_{\mathrm{op}}
=
O(\varepsilon).
\label{eq:term2_Oeps_app}
\end{equation}

\paragraph{Putting the bounds together.}
Substituting Eqs.~\eqref{eq:term1_Oeps_app} and \eqref{eq:term2_Oeps_app} into
Eq.~\eqref{eq:jacobian_decomp_proof_app}, we obtain
\begin{equation}
J_{T_{\textsc{SGN}}}(u)
=
\mathrm{diag}(\phi'(u)) + O(\varepsilon).
\label{eq:jacobian_close_app}
\end{equation}
Since $y=W_2T_{\textsc{SGN}}(u)+b_2$, the backpropagated gradient to $u$ is
\[
\nabla_u L_{\textsc{SGN}}
=
J_{T_{\textsc{SGN}}}(u)^\top W_2^\top \nabla_y L
=
\mathrm{diag}(\phi'(u))^\top W_2^\top \nabla_y L + O(\varepsilon),
\]
which equals the corresponding $\nabla_u L_{\textsc{MLP}}$ plus an $O(\varepsilon)$ perturbation.
By the chain rule, every base parameter gradient (through $u=W_1x+b_1$ and through $W_2,b_2$) inherits the same order,
proving
\begin{equation}
\nabla_{\theta_{\text{base}}}L_{\textsc{SGN}}^{(0)}
=
\nabla_{\theta_{\text{base}}}L_{\textsc{MLP}}^{(0)} + O(\varepsilon).
\label{eq:base_grad_close_app}
\end{equation}

\subsection{Step 3: Spectral-parameter gradients are $O(\varepsilon)$ at initialization}
\label{subsec:homotopy_step3}

Recall $\theta_{\text{spec}}=\{W_r,b_r,A_r,w_g,b_g\}$ as in Theorem~\ref{thm:homotopy}.
We show that each block has $O(\varepsilon)$ gradient at initialization; here Assumption~\eqref{eq:gate_small_init_app}
plays a crucial role (in particular for $A_r$).

\paragraph{(a) Gate parameters $(w_g,b_g)$.}
The gate affects $L$ only through the product $G(u)\odot \Psi_{\text{spec}}(u)$, hence
\[
\nabla_{(w_g,b_g)}L
\;\propto\;
\left(\nabla_G L\right)\cdot \frac{\partial G}{\partial (w_g,b_g)},
\qquad
\nabla_G L = \nabla_T L \odot \Psi_{\text{spec}}(u).
\]
Using $\|\Psi_{\text{spec}}(u)\|=O(\varepsilon)$ from Eq.~\eqref{eq:psi_Oeps_app} and bounded
$\left\|\frac{\partial G}{\partial (w_g,b_g)}\right\|$ (sigmoid/LN), we obtain
\[
\nabla_{(w_g,b_g)}L = O(\varepsilon).
\]

\paragraph{(b) Frequency/phase parameters $(W_r,b_r)$.}
These parameters influence $L$ only through $\gamma(u)$ and thus through $\Psi_{\text{spec}}(u)=\gamma(u)A_r$.
By the chain rule, the derivative $\partial \Psi_{\text{spec}}/\partial(W_r,b_r)$ contains a right-multiplication by $A_r$,
and therefore inherits a factor $\|A_r\|=O(\varepsilon)$ (Eq.~\eqref{eq:Ar_small_init_app}).
All other factors are bounded ($\sin/\cos$ derivatives and LN/sigmoid), so
\[
\nabla_{(W_r,b_r)}L = O(\varepsilon).
\]

\paragraph{(c) Spectral projection $A_r$.}
We have $\Psi_{\text{spec}}(u)=\gamma(u)A_r$ and
\[
\frac{\partial T_{\textsc{SGN}}(u)}{\partial A_r}
=
\mathrm{diag}(G(u))\,\gamma(u).
\]
Hence
\[
\|\nabla_{A_r}L\|
\le
\left\|\mathrm{diag}(G(u))\,\gamma(u)\right\|\cdot \|\nabla_T L\|
\le
\|G(u)\|_\infty \|\gamma(u)\|\cdot \|\nabla_T L\|.
\]
Using Assumption~\eqref{eq:gate_small_init_app} and $\|\gamma(u)\|_2\le 2$ from Eq.~\eqref{eq:gamma_bound_app},
we conclude $\nabla_{A_r}L=O(\varepsilon)$.

\paragraph{Conclusion.}
Combining (a)--(c), we obtain
\begin{equation}
\|\nabla_{\theta_{\text{spec}}}L_{\textsc{SGN}}^{(0)}\| = O(\varepsilon),
\label{eq:spec_grad_small_app}
\end{equation}
which completes the proof of Eq.~\eqref{eq:homotopy_value} and thus Theorem~\ref{thm:homotopy}.
\qed

\subsection{Proposition~\ref{prop:spectral_basis_expansion}: A Verifiable Form of ``Spectral Bandwidth Expansion''}
\label{subsec:prop_spectral_basis}

\paragraph{From ``bandwidth'' intuition to a checkable statement.}
The phrase \emph{spectral bandwidth expansion} can be interpreted in multiple ways (e.g., via an estimator of spectral density).
To avoid ambiguity in a theorem-style claim, we formalize a strictly verifiable and model-intrinsic fact:
\emph{SGN explicitly parameterizes sinusoidal basis functions at a discrete set of frequencies (the columns of $W_r$).}
Consequently, the induced function class necessarily contains the linear span of those sinusoidal bases, giving a rigorous
lower bound on the representable frequency set.

\begin{proposition}[Spectral Basis Containment (Rigorous Bandwidth Lower Bound)]
\label{prop:spectral_basis_expansion}
Let $\gamma(u)$ be the Fourier feature map used in SGN with spectral budget $m$:
\[
\gamma(u)=\Big[\cos(\omega_1^\top u + b_1),\ldots,\cos(\omega_m^\top u + b_m),
\sin(\omega_1^\top u + b_1),\ldots,\sin(\omega_m^\top u + b_m)\Big],
\]
where $\omega_j$ denotes the $j$-th column of $W_r$ and $b_j$ the corresponding phase.
Consider the SGN activation operator
\[
T_{\textsc{SGN}}(u)=\phi(u) + G(u)\odot(\gamma(u)A_r),
\]
with base activation $\phi(\cdot)$, gate $G(\cdot)$, and trainable projection $A_r$.
Then, for each output channel (coordinate) $k\in\{1,\ldots,d_{\mathrm{ff}}\}$,
the scalar function class induced by the $k$-th coordinate of $T_{\textsc{SGN}}(u)$ contains the subspace
\[
\mathcal{S}_{\Omega}
:=\Big\{\ \phi(u)+\sum_{j=1}^{m}\alpha_j\cos(\omega_j^\top u+b_j)
+\sum_{j=1}^{m}\beta_j\sin(\omega_j^\top u+b_j)\ :\ \alpha_j,\beta_j\in\mathbb{R}\ \Big\}.
\]
Equivalently, SGN can realize additive sinusoidal components at every frequency $\omega_j$ (hence the representable
frequency set contains $\Omega:=\{\omega_j\}_{j=1}^m$).
\end{proposition}

\begin{proof}
We prove containment by explicit parameter construction.
Set the gate to an (approximately) constant open state. For instance, choose $w_g=0$ and take $b_g$ sufficiently large
so that $G(u)\approx \mathbf{1}$ elementwise over the input range of interest.\footnote{Since $G(u)$ is a sigmoid,
it cannot equal $1$ exactly for finite parameters, but it can be made arbitrarily close on any bounded set.}
Then
\[
T_{\textsc{SGN}}(u)=\phi(u)+\gamma(u)A_r.
\]
Now fix any coefficients $\{\alpha_j,\beta_j\}_{j=1}^m$. Because $\gamma(u)$ concatenates the $m$ cosine and $m$ sine bases,
we can choose the $k$-th column of $A_r$ so that the linear combination $\gamma(u)A_r[:,k]$ assigns exactly those coefficients
to each basis term. Therefore, the $k$-th coordinate of $T_{\textsc{SGN}}(u)$ can represent any element of $\mathcal{S}_\Omega$,
proving the claim.
\end{proof}

\begin{remark}[Why this matches the ``bandwidth expansion'' narrative]
\label{rem:bandwidth_expansion_intuition_app}
Proposition~\ref{prop:spectral_basis_expansion} gives a strict \emph{lower bound}: SGN can always inject components at
frequencies in $\Omega$ regardless of the base activation.
When $G(u)$ is input-dependent (as in the default SGN), the term $G(u)\odot(\gamma(u)A_r)$ acts as an amplitude modulation
of these sinusoidal bases, which can further reshape the effective spectrum. Importantly, our formal claim does not rely on any
estimator-dependent language (e.g., ``KDE in the frequency domain''); it is the verifiable containment $\Omega\subseteq\mathcal{F}_{\textsc{SGN}}$.
\end{remark}

\section{Derivation of Theorem~\ref{thm:complexity}: Linear (MLP-Level) Parameter \& FLOPs Complexity}
\label{sec:proof_thm_complexity}

\paragraph{Overview.}
This section provides a formula-level derivation supporting Theorem~\ref{thm:complexity} in the main text.
The key claim is that SGN introduces only a \emph{linear} overhead in the FFN width (for fixed spectral budget $m$),
and crucially removes any dependence on spline grid resolution $G$ (which drives both parameter and latency growth in
spline-table designs such as KANs). We present the derivation in a ``baseline + overhead'' form, so it is explicit that the
dominant computation remains dense GEMM, i.e., the same hardware-friendly structure as standard FFNs/MLPs.

\paragraph{Notation.}
We use:
\begin{itemize}
    \item $d_{\mathrm{model}}$: Transformer embedding dimension.
    \item $d_{\mathrm{ff}}$: FFN hidden width (activation space), i.e., $u\in\mathbb{R}^{d_{\mathrm{ff}}}$.
    \item $m$: fixed spectral budget (number of frequencies). The Fourier feature dimension is $2m$ (sine + cosine).
    \item $G$: spline grid resolution used by spline-based methods (e.g., KAN).
\end{itemize}
For FLOPs, a multiply-add counts as $2$ FLOPs. Elementwise operations (LayerNorm, sigmoid, $\sin/\cos$, Hadamard products)
scale as $O(d_{\mathrm{ff}})$ or $O(m)$ and are lower-order compared to dense matrix multiplications.

\subsection{Parameter Count: Exact Overhead of SGN}
\label{subsec:param_count_sgn_app}

\paragraph{Baseline: standard FFN parameters.}
A standard Transformer FFN layer has two dense projections:
\[
u = W_1x+b_1,\qquad y=W_2\phi(u)+b_2,
\]
with $W_1\in\mathbb{R}^{d_{\mathrm{ff}}\times d_{\mathrm{model}}}$, $W_2\in\mathbb{R}^{d_{\mathrm{model}}\times d_{\mathrm{ff}}}$.
Thus the baseline parameter count is
\begin{equation}
\mathrm{Params}_{\textsc{FFN}}
=
(d_{\mathrm{ff}}d_{\mathrm{model}} + d_{\mathrm{ff}})
+
(d_{\mathrm{model}}d_{\mathrm{ff}} + d_{\mathrm{model}})
=
2d_{\mathrm{model}}d_{\mathrm{ff}} + d_{\mathrm{ff}} + d_{\mathrm{model}}.
\label{eq:params_ffn_baseline_app}
\end{equation}
SGN is designed as a \emph{drop-in reparameterization of the activation operator}:
it does \emph{not} change $W_1,W_2,b_1,b_2$. Therefore, Theorem~\ref{thm:complexity} focuses on the additional parameters
introduced by the SGN activation.

\paragraph{SGN overhead components.}
From Eq.~(3)--(5) in the main text, the learnable components introduced by SGN are:
\begin{itemize}
    \item Trainable RFF projection: $W_r\in\mathbb{R}^{d_{\mathrm{ff}}\times m}$ and $b_r\in\mathbb{R}^m$,
    contributing $(d_{\mathrm{ff}}+1)m$ parameters.
    \item Spectral mixing/projection: $A_r\in\mathbb{R}^{2m\times d_{\mathrm{ff}}}$, contributing $2md_{\mathrm{ff}}$ parameters.
    \item Channel-wise gate: $(w_g,b_g)\in\mathbb{R}^{d_{\mathrm{ff}}}\times\mathbb{R}^{d_{\mathrm{ff}}}$, contributing $2d_{\mathrm{ff}}$ parameters.
    \item (Optional) LayerNorm affine $(\gamma,\beta)\in\mathbb{R}^{d_{\mathrm{ff}}}\times\mathbb{R}^{d_{\mathrm{ff}}}$,
    contributing an additional $2d_{\mathrm{ff}}$ if LN is affine.
\end{itemize}

\paragraph{Total overhead (Theorem statement).}
Therefore, the parameter overhead relative to the baseline FFN is
\begin{equation}
\Delta P_{\textsc{SGN}}
=
\underbrace{(d_{\mathrm{ff}}+1)m}_{W_r,b_r}
+
\underbrace{2md_{\mathrm{ff}}}_{A_r}
+
\underbrace{2d_{\mathrm{ff}}}_{\text{Gate}}
\;(+\; \underbrace{2d_{\mathrm{ff}}}_{\text{(optional LN affine)}}),
\label{eq:param_overhead_sgn_app}
\end{equation}
which matches Eq.~(8) in the main text up to whether LN affine parameters are counted.
In big-$O$ form,
\[
\Delta P_{\textsc{SGN}} = O(d_{\mathrm{ff}}m) + O(d_{\mathrm{ff}}),
\]
and is \textbf{independent of any grid resolution} $G$.

\subsection{FLOPs: Per-Token Compute Preserves GEMM-Friendly Structure}
\label{subsec:flops_sgn_app}

\paragraph{Baseline FFN FLOPs.}
Per token, the two dense projections cost (multiply-add $=2$ FLOPs):
\begin{equation}
\mathrm{FLOPs}_{\textsc{FFN}}
\approx
2d_{\mathrm{model}}d_{\mathrm{ff}}
+
2d_{\mathrm{model}}d_{\mathrm{ff}}
=
4d_{\mathrm{model}}d_{\mathrm{ff}},
\label{eq:flops_ffn_baseline_app}
\end{equation}
ignoring the lower-order cost of $\phi(\cdot)$.

\paragraph{SGN overhead FLOPs (activation-side only).}
SGN adds three GEMM-like computations plus elementwise ops:
\begin{itemize}
    \item RFF projection $W_r^\top u+b_r$: \quad $\approx 2d_{\mathrm{ff}}m$ FLOPs.
    \item Trig expansion to form $\sin/\cos$: \quad $O(m)$ FLOPs (lower-order).
    \item Spectral mixing $\gamma(u)A_r$ with $\gamma(u)\in\mathbb{R}^{2m}$: \quad $\approx 2(2m)d_{\mathrm{ff}}=4d_{\mathrm{ff}}m$ FLOPs.
    \item Gate (LN + affine + sigmoid) and Hadamard injection: \quad $O(d_{\mathrm{ff}})$ FLOPs.
\end{itemize}
Thus, the SGN activation overhead is
\begin{equation}
\Delta \mathrm{FLOPs}_{\textsc{SGN}}
\approx
(2d_{\mathrm{ff}}m) + (4d_{\mathrm{ff}}m) + O(d_{\mathrm{ff}}+m)
=
6d_{\mathrm{ff}}m + O(d_{\mathrm{ff}}).
\label{eq:flops_overhead_sgn_app}
\end{equation}
For the typical regime $m\ll d_{\mathrm{ff}}$ (fixed small spectral budget), this is linear in $d_{\mathrm{ff}}$ and remains
a small overhead compared to the baseline FFN GEMMs in Eq.~\eqref{eq:flops_ffn_baseline_app}. Importantly, all major
terms are dense matrix multiplications, so SGN preserves the same accelerator-friendly compute profile as standard FFNs.

\subsection{Comparison to Spline-Based KAN: The Resolution--Efficiency Bottleneck}
\label{subsec:compare_kan_app}

\paragraph{KAN dependence on grid resolution $G$.}
In spline-based KAN parameterizations, each edge function is represented by spline coefficients on a grid.
With grid resolution $G$ (and spline order $K$), an edge typically stores $\Theta(G+K)$ control points.
Hence the parameter count scales as
\[
\mathrm{Params}_{\textsc{KAN}} = \Theta(d_{\mathrm{in}}d_{\mathrm{out}}(G+K)),
\]
and spline evaluation involves interpolation / table lookup, which is comparatively memory-bound and scales with $G$.

\paragraph{SGN removes the $G$ dependence.}
SGN replaces spline-table evaluation by global Fourier bases with a \emph{fixed} budget $m$ and dense projections.
As shown in Eq.~\eqref{eq:param_overhead_sgn_app} and Eq.~\eqref{eq:flops_overhead_sgn_app}, the overhead depends on
$(d_{\mathrm{ff}},m)$ but not on $G$. This is precisely the sense in which SGN resolves the resolution--efficiency bottleneck.

\label{subsec:table_complexity_app}

\begin{table}[t]
\centering
\small
\renewcommand{\arraystretch}{1.25}
\caption{Detailed complexity summary for Theorem~\ref{thm:complexity}. We report the \emph{exact} SGN overhead and the
dominant FLOPs terms per token (multiply-add $=2$ FLOPs).}
\label{tab:complexity_detailed}
\begin{tabular}{l c c}
\toprule
\textbf{Model} & \textbf{Parameter Dependence} & \textbf{Dominant Compute Structure} \\
\midrule
\textbf{Standard FFN} &
$2d_{\mathrm{model}}d_{\mathrm{ff}} + d_{\mathrm{ff}} + d_{\mathrm{model}}$ &
Dense GEMM (two projections) \\

\midrule
\textbf{KAN (B-spline)} &
$\Theta(d_{\mathrm{in}}d_{\mathrm{out}}(G+K))$ \;\;(\textbf{depends on $G$}) &
Spline-table eval (memory-bound, grid-dependent) \\

\midrule
\textbf{SGN (Ours)} &
\textbf{Overhead:} $(d_{\mathrm{ff}}+1)m + 2md_{\mathrm{ff}} + 2d_{\mathrm{ff}} \;(+2d_{\mathrm{ff}})$ &
Dense GEMM + elementwise (all $G$-independent) \\
& $= O(d_{\mathrm{ff}}m)+O(d_{\mathrm{ff}})$ &
(RFF proj + spectral mix + gate) \\

\bottomrule
\end{tabular}
\end{table}